\pdfoutput=1

\documentclass[11pt]{article}

\usepackage{acl2023}

\usepackage{custom}
\usepackage{tabularray}
\usepackage{siunitx}
\usepackage[utf8]{inputenc}

\newcommand{\mymacro}[1]{{#1}}

\newcommand{\defn}[1]{\textbf{#1}}

\newcommand{\pfsaRank}{{\mymacro{\mathrm{R}}}}

\newcommand{\valpha}{{\mymacro{\boldsymbol{\alpha}}}}
\newcommand{\vxi}{{\mymacro{\boldsymbol{\xi}}}}

\newcommand{\pdens}{{\mymacro{ p}}}

\newcommand{\qdens}{{\mymacro{ q}}}

\newcommand{\ind}[1]{\mathbbm{1} \left\{ #1 \right\}}

\newcommand{\R}{{\mymacro{ \mathbb{R}}}}

\newcommand{\RD}{{\mymacro{ \R^D}}}

\newcommand{\norm}[1]{{\mymacro{ \left\lVert #1 \right\rVert}}}

\newcommand{\alphabet}{{\mymacro{ \Sigma}}}

\newcommand{\eosalphabet}{{\mymacro{ \overline{\alphabet}}}}

\newcommand{\kleene}[1]{{\mymacro{#1^*}}}

\newcommand{\str}{{\mymacro{\boldsymbol{y}}}}

\newcommand{\strlt}{{\mymacro{ \str_{<\tstep}}}}

\newcommand{\strlen}{{\mymacro{T}}}

\newcommand{\sym}{{\mymacro{y}}}
\newcommand{\eossym}{{\mymacro{\overline{\sym}}}}
\newcommand{\syma}{{\mymacro{a}}}
\newcommand{\symb}{{\mymacro{b}}}

\newcommand{\defeq}{\mathrel{\stackrel{\textnormal{\tiny def}}{=}}}

\newcommand{\set}[1]{{\mymacro{\left\{ #1 \right\}}}}

\newcommand{\rank}{{\mymacro{\text{rank}}}}
\newcommand{\rankFun}[1]{{\mymacro{\rank\left(#1\right)}}}

\newcommand{\idx}{{\mymacro{ n}}}

\newcommand{\nstates}{{\mymacro{ |\states|}}}
\newcommand{\nsymbols}{{\mymacro{ |\alphabet|}}}
\newcommand{\eosnsymbols}{{\mymacro{ |\eosalphabet|}}}
\newcommand{\tstep}{{\mymacro{ t}}}

\newcommand{\pLM}{\mymacro{\pdens}}
\newcommand{\pLMFun}[1]{\pLM\left(#1\right)}

\newcommand{\qLM}{\mymacro{\qdens}}
\newcommand{\qLMFun}[1]{\qLM\left(#1\right)}

\newcommand{\eos}{{\mymacro{\textsc{eos}}}}

\newcommand{\symt}{{\mymacro{ \sym_{\tstep}}}}

\newcommand{\automaton}{{\mymacro{ \mathcal{A}}}}
\newcommand{\wfsa}{{\mymacro{ \automaton}}}

\newcommand{\stateq}{{\mymacro{ q}}}

\newcommand{\states}{{\mymacro{ Q}}}

\newcommand{\trans}{{\mymacro{ \delta}}}

\newcommand{\weight}{{\mymacro{ \textnormal{w}}}}

\newcommand{\apath}{{\mymacro{ \boldsymbol \pi}}}
\newcommand{\pathlen}{{\mymacro{ N}}}
\newcommand{\paths}{{\mymacro{ \Pi}}}

\newcommand{\initf}{{\mymacro{ \lambda}}}
\newcommand{\finalf}{{\mymacro{ \rho}}}
\newcommand{\initfFun}[1]{{\mymacro{\initf\left(#1\right)}}}
\newcommand{\finalfFun}[1]{{\mymacro{\finalf\left(#1\right)}}}

\newcommand{\transMtx}{{\mymacro{ \mT}}}

\newcommand{\wfsatuple}{{\mymacro{ \left( \alphabet, \states, \trans, \initf, \finalf \right)}}}

\newcommand{\edgenoweight}[3]{#1 \xrightarrow{#2} #3}
\newcommand{\edge}[4]{{\mymacro{#1 \xrightarrow{#2 / #3} #4}}}

\newcommand{\entropy}{{\mymacro{ \mathrm{H}}}}

\newcommand{\yield}{{\mymacro{\textbf{s}}}}

\newcommand{\pfsaAcr}{{\mymacro{PFSA}}\xspace}
\newcommand{\fslmAcr}{{\mymacro{RLM}}\xspace}

\newcommand{\dpfsaAcr}{{\mymacro{DPFSA}}\xspace}

\newcommand{\outMtx}{{\mymacro{ \mE}}}

\newcommand{\hiddDim}{{\mymacro{ D}}}
\newcommand{\Rhid}{{\mymacro{ \R^\hiddDim}}}
\newcommand{\Rhidd}{\Rhid}

\newcommand{\tranMtx}{{\mymacro{\mM}}}

\newcommand{\softmaxfunc}[2]{{\mymacro{ \mathrm{softmax}\!\left(#1\right)_{#2}}}} 
\newcommand{\softmaxFun}[2]{\softmaxfunc{#1}{#2}} 

\newcommand{\hiddState}{{\mymacro{ \vh}}}

\newcommand{\hiddStatetminus}{{\mymacro{ \hiddState_{\tstep - 1}}}}

\newcommand{\mathspan}{{\mymacro{\mathrm{span}}}}
\newcommand{\ones}{{\mymacro{\mathbf{1}}}}

\newcommand{\negterm}[1]{{\mymacro{ {\raise.17ex\hbox{$\scriptstyle\sim$}} #1}}}

\newcommand{\ignore}[1]{}
\newcommand{\expandLater}[1]{}

\newcommand{\transformernetwork}{{\mymacro{ \mathcal{T}}}}

\def\1{\mathbf{1}}

\newcommand{\datasetSize}{{\mymacro{N}}}

\def\vh{{{\mymacro{ \mathbf{h}}}}}

\def\mE{{{\mymacro{ \mathbf{E}}}}}

\def\mI{{{\mymacro{ \mathbf{I}}}}}

\def\mM{{{\mymacro{ \mathbf{M}}}}}

\def\mT{{{\mymacro{ \mathbf{T}}}}}

\newcommand{\E}{{\mymacro{ \mathbb{E}}}}

\newcommand{\KL}{{\mymacro{ \mathrm{D}_{\mathrm{KL}}}}}
\newcommand{\KLFun}[2]{\KL\left(#1 \mid \mid #2\right)}

\DeclareMathOperator*{\argmax}{{\mymacro{ argmax}}}
\DeclareMathOperator*{\argmin}{{\mymacro{ argmin}}}

\DeclareMathSymbol{\mlq}{\mathord}{operators}{``} 
\DeclareMathSymbol{\mrq}{\mathord}{operators}{`'} 

\newcommand{\cop}{1}
\newcommand{\ethz}{2}
\newcommand{\uzh}{3}

\title{What Languages are Easy to Language-Model? A Perspective from Learning Probabilistic Regular Languages}

\author{
 Nadav Borenstein$^{\cop}$ ~\;~
 Anej Svete$^{\ethz}$ ~\;~
 \textbf{Robin Shing Moon Chan}$^{\ethz}$ ~\;~
 \textbf{Josef Valvoda}$^{\cop}$\\
   \textbf{Franz Nowak}$^{\ethz}$~\;~
   \textbf{Isabelle Augenstein}$^{\cop}$~\;~
   \textbf{Eleanor Chodroff}$^{\uzh}$~\;~
  \textbf{Ryan Cotterell}$^{\ethz}$
\\
 $^{\cop}$Københavns Universitet~\;~  $^{\ethz}$ETH Zürich~\;~ 
 $^{\uzh}$Universit{\"a}t Zürich\\
 \{\texttt{\href{mailto:nb@di.ku.dk}{nb}},
 \texttt{\href{mailto:jval@di.ku.dk}{jval}}, \texttt{\href{mailto:augenstein@di.ku.dk} {augenstein}}\}\texttt{@di.ku.dk} \quad \texttt{\href{mailto:eleanor.chodroff@uzh.ch}{eleanor.chodroff@uzh.ch}}  \\
 \{\texttt{\href{mailto:asvete@inf.ethz.ch}{asvete}},
 \texttt{\href{mailto:robin.chan@inf.ethz.ch}{robin.chan}}, \texttt{\href{mailto:fnowak@inf.ethz.ch}{fnowak}}, \texttt{\href{mailto:ryan.cotterell@inf.ethz.ch}{ryan.cotterell}}\}\texttt{@inf.ethz.ch}
}
\begin{document}

\maketitle

\begin{abstract}
    What can large language models learn?
    By definition, language models (LM) are distributions over strings.
    Therefore, an intuitive way of addressing the above question is to formalize it as a matter of learnability of \emph{classes} of distributions over strings.
    While prior work in this direction focused on assessing the theoretical limits, in contrast, we seek to understand the empirical learnability.
    Unlike prior empirical work, we evaluate neural LMs on their home turf---learning probabilistic languages---rather than as classifiers of formal languages.
    In particular, we investigate the learnability of regular LMs (\fslmAcr{}s) by RNN and Transformer LMs.
    We empirically test the learnability of \fslmAcr{}s as a function of various complexity parameters of the \fslmAcr{} and the hidden state size of the neural LM. 
    We find that the \fslmAcr{} rank, which corresponds to the size of linear space spanned by the logits of its conditional distributions, and the expected length of sampled strings are strong and significant predictors of learnability for both RNNs and Transformers. 
    Several other predictors also reach significance, but with differing patterns between RNNs and Transformers.\looseness-1
\end{abstract}

\section{Introduction}\label{sec:intro}

Language models are, definitionally, distributions over strings.
However, not all neural LMs are capable of learning---or even representing---all possible distributions.
This raises two natural questions: What classes of distributions \emph{can} neural LMs represent and what can they learn from training examples?
In terms of the first question, the relationship between recurrent neural networks and more symbolic computational models has been subject to study for over three decades \citep{McCulloch1943,Kleene1956,siegelmann-sontag-1992,hao-etal-2018-context,DBLP:journals/corr/abs-1906-06349,merrill-2019-sequential,merrill-etal-2020-formal,hewitt-etal-2020-rnns,Chung2021,merrill-etal-2022-saturated,merrill2022extracting,svete2023recurrent,nowak-etal-2023-representational}.
Moreover, the prevalence of Transformer-based LMs has led to a recent body of work investigating their representational capacity \citep[e.g.,][]{hahn-2020-theoretical,ebrahimi-etal-2020-self,bhattamishra-etal-2020-ability,merrill-sabharwal-2023-parallelism}.
However, almost all of this work is theoretical, i.e., researchers seek theorems that give exact limitations on the capacity of specific neural LMs.
While such work provides a good characterization of what neural LMs could, in principle, learn, it does not speak to what LMs can learn in practice.\looseness=-1

\begin{figure}[t]
    \centering
    \includegraphics[width=\columnwidth]{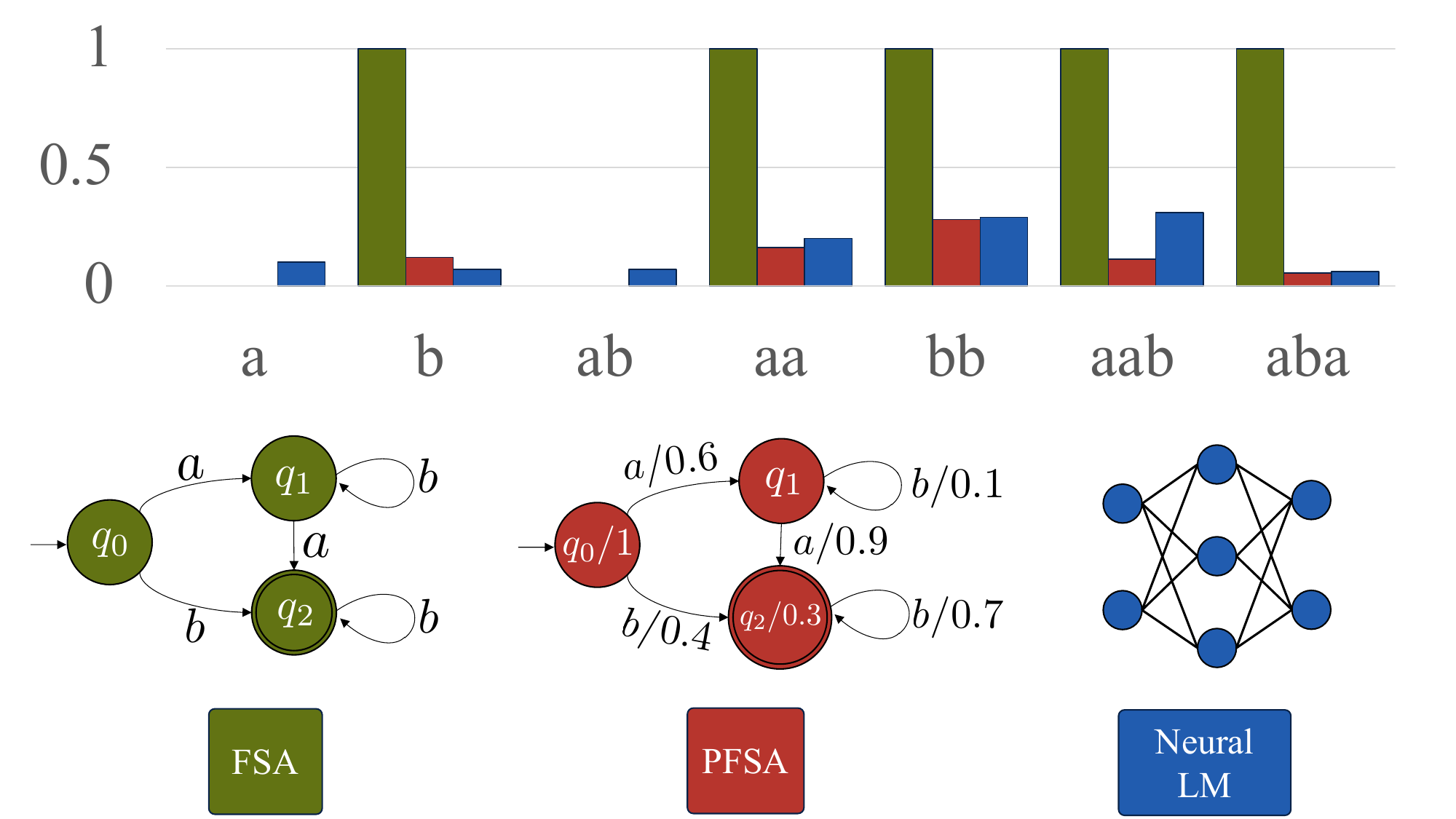}
    \caption{A \textcolor{ETHGreen}{finite-state automaton} defines a \emph{set} of strings by assigning string binary weights. 
    A \textcolor{ETHRed}{probabilistic finite-state automaton} and a \textcolor{ETHBlue}{neural LM} such as an RNN or a Transformer LM, however, define a \emph{probability distribution} over strings. 
   }
    \label{fig:example-fslm}
\end{figure}

In contrast to the above, a more empirically minded researcher might prefer to run a series of controlled experiments to determine what neural LMs can and cannot learn.
Their goal is to empirically characterize what classes of formal LMs, e.g., probabilistic finite-state automata, \emph{are} learnable with neural LMs in practice, using current best practices.
Such work can inform our understanding of what types of languages larger LMs trained on human-written text might represent---specifically, what grammatical structures they can recognize, and how efficiently they can do so.
All of the above is crucial for quantifying the practical capabilities, and limits, of neural LMs.
While plenty of empirical work has provided insights into the linguistic capabilities of modern LMs using linguistically annotated datasets \citep[e.g., ][]{10.1162/tacl-a-00115,hewitt-manning-2019-structural,jawahar-etal-2019-bert,liu-etal-2019-linguistic,ICARD2020102308,doi:10.1073/pnas.1907367117,10.1162/tacl-a-00349,belinkov-2022-probing}, human-annotated datasets give us limited insight into the types of distributions a neural LM can learn because the true distribution that the neural LMs is modeling is often unclear.
For instance, fitting an LM to Wikipedia leaves open to interpretation exactly \emph{which} probability distribution over strings the neural LM is modeling.
In contrast, learning a probabilistic formal language in a controlled experiment offers an unparalleled level of control.\looseness-1

A closer look at the existing work on the empirical learnability of formal languages (see \Cref{app:related_work} for an overview) reveals a categorical mismatch between what LMs are, i.e., \emph{probability distributions} over strings, and what learning a formal language means, i.e., classifying strings as members of a specific language, i.e., a \emph{set} of strings \cite{ebrahimi-etal-2020-self,deletang2023neural,wang-steinert-threlkeld-2023-evaluating}.
In response, we propose to investigate the practical representation capacity of neural LMs by testing their ability to learn \emph{distributions} over strings.
By sampling training corpora from probabilistic finite-state automata (\pfsaAcr), we can ask precise questions about the learnability.

We offer an empirical study, sampling datasets of 20k strings from 2100 randomly generated \pfsaAcr{s}, and training 15k RNN and Transformer language models with a varying hidden state size on datasets.
Our study is informed by various theoretical results on the representational capacity of RNNs concerning probabilistic finite-state automata. 
We assess the learnability by approximating the KL divergence between \pfsaAcr{s} and neural LMs.
In a regression analysis, we find that a large number of properties of the automaton, e.g., the number of states, the number of transitions, the rank of its emission matrix, and its entropy all contribute to learnability. 
In particular, the rank of the emission matrix and expected length of the sampled strings are strong predictors of learnability for both RNN and Transformer LMs. Several other predictors also demonstrate relevance in accounting for learnability, but have differing magnitudes and patterns of significance for the RNN and Transformer LMs. 
Similar to \citet{deletang2023neural}, we find that RNNs are better suited to modeling formal languages than Transformers.\looseness=-1


\section{Preliminaries} \label{sec:preliminaries}
We begin with an introduction of the relevant mathematical preliminaries, based on \citet{cotterell2024formal}.
An \defn{alphabet} $\alphabet$ is a finite, non-empty set of \defn{symbols}.
The \defn{Kleene closure} $\kleene{\alphabet}$ of the alphabet $\alphabet$ is the set of all strings of the symbols in $\alphabet$.
The \defn{length} of the string $\str = \sym_1\ldots\sym_\strlen \in \kleene{\alphabet}$, denoted by $|\str|=\strlen$, is the number of symbols the string contains.
A \defn{language model} $\pLM$ is a probability distribution over $\kleene{\alphabet}$.
Two LMs $\pLM$ and $\qLM$ are \defn{equivalent} if $\pLM\left(\str\right) = \qLM\left(\str\right)$ for all $\str \in \kleene{\alphabet}$.
Most modern LMs define $\pLM\left(\str\right)$ as a product of conditional probability distributions:
\begin{equation} \label{eq:lnlm}
    \pLM\left(\str\right) \defeq \pLM\left(\eos\mid\str\right) \prod_{\tstep = 1}^{|\str|} \pLM\left(\symt \mid \strlt\right),
\end{equation}
where $\eos \notin \alphabet$ is a special \underline{e}nd-\underline{o}f \underline{s}equence-symbol.
We denote $\eosalphabet \defeq \alphabet \cup \left\{\eos\right\}$ and $\overline{\sym}$ an element of $\eosalphabet$.

\subsection{Neural Language Models}

Representation-based neural LMs define the conditional distributions $\pLM\left(\overline{\sym}_t \mid \strlt\right)$ through a linearly transformed \defn{hidden state}.
In this paper, we focus on \defn{softmax-normalized}, \defn{representation-based} neural LMs, where a \defn{language encoder} $\hiddState \colon \kleene{\alphabet} \rightarrow \RD$ \citep{chan2024affine} computes a representation, and the conditional distributions of the neural LMs are defined as follows
\begin{subequations}
\begin{align}
    \pLM\left(\overline{\sym}_t \mid \strlt\right) & \defeq \softmaxFun{\outMtx \hiddState(\strlt)}{\overline{\sym}_t} \label{eq:neural_lm_conditinal} \\ &\defeq \frac{\exp\left(\outMtx\hiddState(\strlt)\right)_{\overline{\sym}_t}}{\sum_{\overline{\sym} \in \eosalphabet} \exp\left(\outMtx\hiddState(\strlt)\right)_{\overline{\sym}}}.
\end{align}
\end{subequations}
We will call $\outMtx \in \R^{\eosnsymbols \times \hiddDim}$ the \defn{output matrix}.

Representation-based neural LMs differ in how $\hiddState$ is computed as a function of $\strlt$.
In this paper, we consider the two most popular modern language modeling architectures: recurrent neural networks \citep{Elman1990}, specifically the LSTM variant \citep{10.1162/neco.1997.9.8.1735}, where $\hiddState$ is computed recurrently, and Transformers \citep{NIPS2017_3f5ee243}, where $\hiddState$ is computed with self-attention.\looseness=-1

\subsection{Regular Language Models}
A classic formalism for defining LMs is \defn{probabilistic finite-state automata} (\pfsaAcr{}s), a probabilistic version of finite-state automata that defines string probabilities.
Intuitively, a \pfsaAcr defines a \emph{finite} number of conditional next-symbol distributions $\pLM\left(\overline{\sym} \mid \stateq\right)$ based on a finite number of states $\stateq \in \states$ that summarize string prefixes analogous to how the hidden state $\hiddState$ of an RNN summarizes the prefix $\sym_1 \ldots \symt$.\looseness=-1

A \pfsaAcr moves between its states based on the input symbols according to the transitions defined by a transition relation.
It \defn{accepts} a string with the probability equal to the product of the transition weights along the string's path in the automaton and the last state's final weight (or the sum over all paths if there are multiple paths accepting the string).\footnote{Final weights of states are analogous to the $\eos$ symbol which signals the end of string generation in neural LMs.}
The \pfsaAcr is \defn{deterministic} (a \dpfsaAcr{}) if the transition relation is a \emph{function} of the current state and symbol, i.e., if, for all $\stateq \in \states$ and $\sym \in \alphabet$, there exists at most one $\stateq^\prime \in \states$ such that $\pLM\left(\stateq^\prime \mid \stateq, \sym\right) > 0$.
A \dpfsaAcr{} is \defn{minimal} if it has the smallest number of states among all its equivalent \dpfsaAcr{}{}s.\footnote{See \cref{sec:pfsas} for more formal details.}
The minimal \dpfsaAcr{} is unique up to the naming of the states. 

\begin{definition}
    The LM $\pLM$ is \defn{regular} (an \fslmAcr) if there exists an equivalent \pfsaAcr.
\end{definition}
\cref{fig:example-fslm} shows an example of an \fslmAcr defining a distribution over $\kleene{\set{\syma, \symb}}$ with $\pLM\left(\syma \symb^n \syma \symb^m\right) = \textcolor{ETHGreen}{1} \cdot 0.6 \cdot 0.1^n \cdot 0.9 \cdot 0.7^m \cdot \textcolor{ETHRed}{0.3}$ and $\pLM\left(\symb \symb^m\right) = \textcolor{ETHGreen}{1} \cdot 0.4 \cdot 0.7^m \cdot \textcolor{ETHRed}{0.3}$.


\section{Representing \fslmAcr{}s with Neural LMs} 
\label{sec:theory}
Neural LMs have demonstrated an ability to model human language well. This raises the question of what enables their success in capturing human language and how we can improve them further.
However, neural LMs are notoriously challenging to analyze, making it difficult to state any formal claims on what they are (in)capable of modeling.
To amend this, a large body of work has linked neural LMs to formal models of computation.
\dpfsaAcr{}s feature prominently in this line of research \citep{merrill-2019-sequential,merrill-etal-2020-formal,svete2023recurrent}.
To facilitate a detailed inspection of how neural LMs can represent \dpfsaAcr{}s, we now formalize \dpfsaAcr{}s in a way that is particularly easy to connect to softmax-normalized representation-based neural LMs, see \cref{sec:pfsas} for the definition of \dpfsaAcr{}.\looseness-1

\subsection{Softmax-Normalized \dpfsaAcr{}s}
\label{sec:softmax-dpfsa}
The conditional distributions $\pLM\left(\overline{\sym}\mid\stateq\right)$ defined by a \dpfsaAcr{} can, in general, be arbitrary distributions over $\eosalphabet$---a \dpfsaAcr{} therefore defines $\nstates$ distributions, each with $\eosnsymbols - 1$ degrees of freedom.
As we will see, such a parameterization makes the connection to neural LMs, which define conditional distributions in terms of shared parameters of the neural network and the output matrix $\outMtx$, straightforward. 
To make this connection clear, consider a \dpfsaAcr{} $\automaton$ with support over $\kleene{\alphabet}$.
Because $\automaton$ is deterministic, there exists a function $\hiddState_{\automaton} \colon \kleene{\alphabet} \rightarrow \{0, 1\}^{\nstates}$ that maps every string that the unique state $\automaton$ enters to a one-hot encoding after reading in the string.
Encoding the final weight as $\eos$, we end up with the following autoregressive formulation of $\automaton$,\looseness=-1
\begin{equation}\label{eq:rep-based-dpfsa}
\pLM\left(\overline{\sym}_t \mid \strlt \right) \defeq \softmaxFun{\transMtx \hiddState_\automaton(\strlt)}{\overline{\sym}_t}.
\end{equation}
where $\transMtx \in \R^{\eosnsymbols \times \nstates}$.
Note that it is because we sought to use a softmax that we require support over all of $\alphabet^*$.
Also, despite its deceiving notation, \cref{eq:rep-based-dpfsa}, is not, in this form, obviously a neural LM as we did not give a neural network that computes $\hiddState_\automaton$---its computation is performed by a \dpfsaAcr{}.
Written as in \cref{eq:rep-based-dpfsa}, we say that a \dpfsaAcr{} is of rank $\pfsaRank$ if the output matrix $\transMtx$ is of rank $\pfsaRank$.\looseness=-1


\paragraph{Minimal representations.}
If $\automaton$ is a minimal \dpfsaAcr{}, then the linear subspace spanned by $\{\hiddState_{\automaton}(\str) \mid \str \in \kleene{\alphabet} \}$ will be a subspace of $\R^{\nstates}$ of dimensionality $\nstates$ due to the one-hot nature of the state encodings. 
However, if $\transMtx$ is of rank $\pfsaRank$, then
the linear subspace $\{\transMtx\hiddState_{\automaton}(\str) \mid \str \in \kleene{\alphabet} \}$ cannot span a subspace of dimensionality larger than $\pfsaRank$. 
Therefore, for a neural LM to be equivalent to the \dpfsaAcr{}, the hidden state's dimension must be large enough. 
If that is not the case, the neural LM will naturally \emph{not} be able to match all the conditional distributions defined by the states of the \dpfsaAcr, which leads us to the following result.\looseness=-1
\begin{restatable}{reTheorem}{minSizeThm} \label{thm:rank}
    Let $\pLM$ be a language model induced by a minimal \dpfsaAcr{} $\automaton$ with full support.
    Furthermore, let $\pfsaRank$ be the rank of $\pLM$ when expressed as a representation-based LM (\cref{eq:rep-based-dpfsa}).
    Then, \emph{any} equivalent representation-based LM $\qLM$ must have a hidden state of size at least $\pfsaRank$.\looseness=-1
\end{restatable}
\begin{proof}
    See \cref{sec:proofs}.
\end{proof}
Given a rank-$\pfsaRank$ \dpfsaAcr $\wfsa$, \cref{thm:rank} says that an equivalent neural LM needs a hidden state of size at least $\pfsaRank$, establishing a general lower bound on an LM's hidden state size for equivalence with a \dpfsaAcr{}. 
A hidden state of size $\pfsaRank$, however, does not \emph{guarantee} that the neural LM can implement the transitions between \dpfsaAcr{}'s states.
The neural LM's ability to do so depends on the particular architecture implementing the neural LM.
In other words, \cref{thm:rank} furnishes us with a necessary condition. 
In the rest of the section, we briefly recapitulate some known results on the relationship between \dpfsaAcr{}s and the two neural architectures in question---recurrent neural networks and Transformers.

\subsection{\dpfsaAcr{}s and Recurrent Neural LMs}
RNNs' connection to \dpfsaAcr{}s is particularly well understood for recurrent neural LMs \citep[e.g.,][]{peng-etal-2018-rational,merrill-2019-sequential,merrill-etal-2020-formal,svete2023recurrent,svete-etal-2024-lower}.
The following theorem from \citet{svete2023recurrent} very concretely summarizes the relationship between \dpfsaAcr{}s and finite-precision RNNs, i.e., ones where the representations $\hiddState$ fall into some finite subset of $\Rhidd$.
\begin{restatable}[\citet{svete2023recurrent}, Thm. 4.1]{reTheorem}{oldTheoremOne} \label{thm:restatement-1}
    The classes of finite-precision RNN LMs and \dpfsaAcr{} are equivalent.
\end{restatable}
While \cref{thm:restatement-1} ensures the existence of a finite-precision RNN LM weakly equivalent to a given \dpfsaAcr{}, it does not describe the size of the hidden state required for doing so.
The following theorem makes the relationship more precise: It states that, in general, the size of the hidden state of RNN LMs must scale linearly with the number of states if the \dpfsaAcr{} is full-rank, i.e., of rank $\nstates$.\looseness=-1
\begin{restatable}[\citet{svete2023recurrent}, Thms. 5.1 and 5.2]{reTheorem}{oldTheoremTwo} \label{thm:restatement-2}
A family of \dpfsaAcr{}s exists such that the hidden states sizes of equivalent RNN LMs scale linearly with $\nstates$ and $\nsymbols$.\looseness=-1
\end{restatable}
Our \cref{thm:rank} can be seen as a more precise version of \cref{thm:restatement-2}, which does not consider the rank of the \dpfsaAcr{}.\footnote{The linear scaling with respect to $\nstates$ is unique to weighted FSAs; more efficient constructions simulating \emph{unweighted} FSAs with RNNs exist \citep{Indyk95,svete2023recurrent}, highlighting the importance of considering probabilistic models separately.}

\subsection{\dpfsaAcr{}s and Transformer LMs}
Statements similar to \cref{thm:restatement-1} are more difficult to make for Transformer LMs due the their parallelizable nature \citep{merrill-sabharwal-2023-parallelism,svete2024transformers}; unlike \dpfsaAcr{}s, Transformers do not compute a sequential inner state.
Existing work has connected Transformers to sequential automata like \dpfsaAcr{}s with the empirically successful framework of chain-of-thought reasoning \citep{perez-turing,feng2023revealing,merrill2023expressive} or by assuming a growing number of layers with increasing string length \citep{liu2023transformers}.
However, these results concern binary language recognition rather than the task of language models. To the best of our knowledge, the only work that directly involves language modeling setting connects Transformers to the simple $n$-gram LMs, a special, well-structured class of \dpfsaAcr{}s.\looseness=-1
\begin{restatable}[\citet{svete2024transformers}, Thm 3.1]{reTheorem}{oldTheorem3}\label{thm:restatement-3}
    For any $n$-gram LM $\pLM$, there exists an equivalent Transformer LM $\pLM_\transformernetwork$ with hard attention.\looseness=-1
\end{restatable}
Unlike \cref{thm:restatement-1}, \cref{thm:restatement-3} does not provide a complete characterization of Transformer LMs but rather a (loose) lower bound on their capabilities, in line with their weaker connection to sequential automata.
We are not aware of any claims analogous to \cref{thm:restatement-2} for Transformer LMs; this makes the generally applicable \cref{thm:rank} that much more interesting, as it lower-bounds the size of the vectorial representations regardless of the neural architecture used to compute them.
This distinction is also mirrored in the empirical part of the paper.\looseness=-1

\subsection{Beyond Representational Capacity}
The theoretical results presented and summarized here capture the fact that the size of the hidden representations in neural LMs must inevitably scale with the size of the \dpfsaAcr{} being simulated or learned (either with its rank, its number of states, or the size of the alphabet).
In that sense, the representational capacity of neural LMs can be theoretically described relatively comprehensively in terms of formal models of computation.
However, existing theoretical work only considers the question of which distributions can be \emph{represented} by a neural LM. 
This leaves us with a large gap in understanding which distributions are \emph{learnable} by neural LMs.
Compared to pure representational capacity results, formal claims about learning are much more difficult to make due to the dependence on factors such as the learning algorithm and aspects of the training data. Therefore, we test the learnability of \dpfsaAcr{}s \textit{empirically}, allowing us to overcome the above-mentioned difficulties and gain valuable insights into the problem.

\section{Practical Learnability of \fslmAcr{}s}
Our main goal is to provide a principled study of the ability of neural LMs to learn \fslmAcr{}s.
We now describe and justify our experimental design.

\subsection{A Critique of Learning Formal Languages}
As discussed in \cref{sec:intro}, much empirical work has investigated the ability of neural language models to learn formal languages such as those described by finite-state automata, i.e., how well a neural language model can be used to assess membership of individual strings in a set. 
There are multiple workarounds for this discrepancy.
Most solutions involve measuring some sort of \emph{accuracy} of next-symbol prediction.
For example, \citet{suzgun-etal-2019-lstm,suzgun-etal-2019-evaluating} and \citet{bhattamishra-etal-2020-ability} evaluate neural LMs on the next-symbol prediction task, which, intuitively, measures whether all allowed continuations of the string under the formal model achieve a large enough probability under the neural LM.
\citet{deletang2023neural} evaluate the models by the proportion of correctly predicted tokens (where the approximate $\argmax$ of the neural LM prediction has to match the ground-truth label).
Unfortunately, all these approaches inevitably shoehorn a neural LM into a sort of classifier, mismatching the type of an LM---a probability distribution---and a formal language---a set, as illustrated in \cref{fig:example-fslm}.
Ideally, we would like to measure precisely how the neural LM has learned the \emph{distribution} induced by a formal LM.
In this section, we outline and motivate a possible way to approach this challenge.


\begin{table}
    \centering
    \fontsize{9}{9}\selectfont
    \begin{tabular}{p{0.12\textwidth}p{0.305\textwidth}}
    \toprule
    Predictor & Interpretation \\
    \midrule
    $\nstates$    &  The number of states \\
    $\nsymbols$    & The number of symbols \\
    $\nstates\nsymbols$ & The number of transitions\\
    $\pfsaRank$ & The rank of the output matrix $\outMtx$ \\
    Exp. length & Expected length of strings sampled from \pfsaAcr \\
    $\min(\nstates, \eosnsymbols)$ & (Tight) upper bound on $\pfsaRank$ \\
    $\entropy(\wfsa)$ & Entropy of  \pfsaAcr{}
    \\ \midrule 
     $\hiddDim$ & The hidden state size of the Elman RNN or Transformer \\
    \bottomrule
    \end{tabular}
    \caption{The predictor variables used to estimate KL divergence with their interpretation: the \pfsaAcr{}-related properties are listed first, followed by those of the neural LM ($\hiddDim$).}
    \label{tab:axes}
\end{table}

At a high level, we test the learnability of \emph{random} representation-based \fslmAcr{}s by training neural LMs on strings sampled from randomly generated \dpfsaAcr{}s and measuring the distance between the neural LM and the \dpfsaAcr{}.
Crucially, unlike most existing work, we do not have to rely on classification or next-symbol prediction accuracy-based metrics, but rather we directly measure the similarity of distributions which presents a much cleaner way of evaluating model similarity.
Concretely, given an \dpfsaAcr $\pLM$ and a neural LM $\qLM$, we measure the KL divergence between the \dpfsaAcr and the neural LM:
\begin{subequations}
\begin{align}
    \KLFun{\pLM}{\qLM} &\defeq \sum_{\str \in \kleene{\alphabet}} \pLMFun{\str} \log\frac{\pLMFun{\str}}{\qLMFun{\str}} \\
    &= \entropy(\pLM, \qLM) - \entropy(\pLM) \label{eq:entropy-diff}.
\end{align}
\end{subequations}
The KL divergence is an established and well-understood measure of the distance\footnote{Note, that KL divergence is not a \emph{true} distance, as it is not symmetric and does not fulfill the triangle inequality.} between two \emph{distributions}. 
As such, it lends itself naturally to evaluating the difference between LMs; in our case, measuring how well the neural LM has captured the distribution of the \dpfsaAcr{}.
Such a holistic treatment of the difference between two LMs gives us a tangible and interpretable way of understanding how they differ.
To compute $\KLFun{\pLM}{\qLM}$, we use \Cref{eq:entropy-diff}. 
We estimate the first term $\entropy\left(\pLM, \qLM\right)$ by computing $\widehat{\entropy}\left(\pLM, \qLM\right)$, the empirical cross-entropy between $\pLM$ and $\qLM$.
The second term can be computed exactly by dynamic programming \citep{eisner-2002-parameter,zmigrod-etal-2021-efficient-computation}.\footnote{\textcolor{brickred}{$\KLFun{\pLM}{\qLM}$ can therefore be computed exactly if $\pLM$ and $\qLM$ are both DPFSAs. When $\qLM$ is a neural LM, however, the finite-sample approximation of $\KLFun{\pLM}{\qLM}$ is required. An issue may arise if $\KLFun{\pLM}{\qLM} = \infty$ since that would imply we are approximating $\infty$ with a finite sum. Nevertheless, we do not observe any signs of divergence of the finite approximations. Moreover, since the approximation comes from computing $\widehat{\entropy}\left(\pLM, \qLM\right)$, which, in our case, is exactly the objective minimized during the training of $\qLM$ (albeit on a different set of strings), we argue that $\entropy\left(\pLM, \qLM\right)$ is unlikely to be infinite.}}
See \cref{app:evaluation_setup} for further details on the computation of these evaluation metrics.\looseness=-1

\begin{table}
    \centering
    \fontsize{10}{10}\selectfont
    \sisetup{table-format=3.2, group-minimum-digits=3}
    \begin{tabular}{lrrr}
        \toprule
        Predictor & $\widehat{\beta}$ & SE & $p$-value \\
        \midrule
        Intercept & $8.72$ & $0.08$ & $<0.001$ \\
        $\nstates$ & $0.68$ & $0.15$ & $<0.001$ \\
        $\nsymbols$ & $0.22$ & $0.20$ & $0.26$ \\
        $\nstates \nsymbols$ & $0.23$ & $0.13$ & $0.07$ \\
        $\pfsaRank$ & $4.10$ & $0.10$ & $<0.001$ \\
        Exp. len. & $3.21$ & $0.19$ & $<0.05$ \\
        $\min(\nstates, \nsymbols)$ & $-0.15$ & $0.26$ & $0.58$ \\
        $\entropy(\wfsa)$ & $-0.88$ & $0.22$ & $<0.001$ \\
        $\hiddDim$ & $-0.63$ & $0.08$ & $<0.001$ \\
        \bottomrule
    \end{tabular}
    \caption{Estimated coefficients ($\widehat{\beta}$), standard errors (SE), and $p$-values for $\KL$ generated with a linear regression model for RNNs.\looseness=-1}
    \label{tab:rnns_coefficients}
\end{table}


\subsection{Generating Random \dpfsaAcr{}s}
We evaluate neural LMs' ability to learn \emph{random} representation-based \dpfsaAcr{}s, which we construct as follows.
First, we sample the number of states $\nstates \in \left\{2, 4, 6, 8, 10, 12, 16\right\}$ and the number of symbols $\nsymbols \in \left\{2, 4, 6, 8, 10, 12, 16\right\}$, both uniformly at random.
Then, we select the destination for the outgoing arcs of each of the states---one for each $\sym \in \alphabet$, i.e., for each $\stateq \in \states$ and $\sym \in \alphabet$ we randomly choose $\stateq{}' \in \states$ and add the transition $\edgenoweight{\stateq}{\sym}{\stateq{}'}$ to $\wfsa$.
Again, the sampling is done uniformly at random. 
We add weights to the transition function of $\wfsa$ as follows. 
First, we generate a random matrix $\transMtx \in \R^{|\eosalphabet| \times \nstates}$. 
For each $\pfsaRank \in \{1, 2, 4, 6, 8, 10, 12, 16\}$ satisfying $\pfsaRank \le \min\left(\nstates, \eosnsymbols\right)$, we define 
$\transMtx^\pfsaRank$ to be the best rank-$\pfsaRank{}$ approximation of $\transMtx{}$, i.e., $\transMtx^\pfsaRank = \argmin_{\mM} \norm{\mM - \transMtx}^2_{\mathrm{F}}$ subject to the constraint $\text{rank}(\mM) = \pfsaRank$.
This is an example of the Procrustes problem and can be efficiently computed using SVD.
We then set the transition probability of $\edgenoweight{\stateq}{\sym}{\stateq{}'}$ to $\weight_{\stateq, \sym} = \softmaxFun{\transMtx{}^\pfsaRank{}_{:, \stateq}}{\sym}$. 
Similarly, we set the final weight $\finalf \left( \stateq \right) = \softmaxFun{\transMtx{}^{\pfsaRank{}}_{:, \stateq}}{\eos{}}$.\looseness=-1

As the Transformer LMs have a limited context length, it is undesirable to sample \dpfsaAcr{}s that generate strings that are longer than the context length with high probability. 
This would introduce an artificial discrepancy between the entropy of the \dpfsaAcr{}s and the empirical entropy of the truncated strings used to train and evaluate the Transformer models. Therefore, we filter out \dpfsaAcr{}s with high expected sequence length. In practice, we filter out \dpfsaAcr{}s with an expected length larger than the median value of expected lengths, i.e., half of the \dpfsaAcr{}s.
See \Cref{app:evaluation_setup} for details about calculating the expected string length of a \dpfsaAcr{}.\looseness=-1

This process results in the generation of up to eight random \dpfsaAcr{}s,
\footnote{We have that $|\{r \mid r \le \min\left(\nstates, \nsymbols\right), r \in \{1, 2, 4, 6,$ $ 8, 10, 12, 16\} \}| \le 8$}
all sharing the same $\states, \alphabet$, and underlying transition function. 
They differ, however, in the rank of the matrix $\transMtx{}^\pfsaRank{}$ that defines the weights of the transitions. 
Furthermore, the construction of exactly one transition for each $\stateq$ and $\sym$ ensures that the \dpfsaAcr{} mirrors the nature of a neural LM, which also defines full-support next-symbol probabilities for any prefix of the string.
Altogether, this allows us to precisely control the quantities from \cref{tab:axes} and thus the complexity of the \dpfsaAcr{}.
See \cref{app:sampling_pdfas} for additional details.

\section{Results}

\setlength{\tabcolsep}{5.5pt}
\begin{table}
    \centering
    \fontsize{9}{10}\selectfont
    \sisetup{table-format = 3.2, group-minimum-digits=3}
    \begin{tabular}{lrrrrrrr}
        \toprule
        \multirow{2}{*}{$\KL$} & \multicolumn{7}{c}{$\nstates$} \\  \cmidrule{2-8}
         & 2 & 4 & 6 & 8 & 10 & 12 & 16 \\
        \midrule
        RNNs & 2.6 & 6.9 & 8.3 & 7.7 & 8.6 & 10.7 & 11.1 \\
        Trans. & 8.8 & 13.7 & 14.4 & 11.9 & 13.5 & 14.4 & 15.3 \\
        \bottomrule
    \end{tabular}

    \caption{KL divergence for RNNs and Transformers as a function of the number of states of the \dpfsaAcr{}.}
    \label{tab:general_nstates}
\end{table}

\subsection{Statistical Evaluation}\label{sec:mixed_effect}
Following the experimental setup outlined in \cref{sec:experimental-details}, we obtained $\KL$ results for 15k RNN and 15k Transformer LMs of differing hidden state dimensions, trained on strings sampled from 2100 random \dpfsaAcr{}s with specific sets of complexity parameters. 
The complexity parameters of \dpfsaAcr{}s correspond to the number of states, number of symbols, and the \dpfsaAcr{} rank; additional derivative metrics can also capture overall complexity, including the number of transitions, the expected string length, the upper bound of the \dpfsaAcr{} rank, and the \dpfsaAcr{} entropy (see \cref{tab:axes}).
We expect the difficulty of learning, quantified as KL divergence, to increase with each complexity parameter and to decrease with the hidden state size.\looseness=-1

As shown in Fig. \ref{fig:rank_states}, as the number of states and \dpfsaAcr{} rank increase, overall learning difficulty also empirically increases for both RNNs and Transformers. To demonstrate that the observed effect is not wholly reducible to differences in entropy, the relationship between the combined number of states and rank with \dpfsaAcr{} entropy is also presented. (Note that the KL divergence is the LM loss minus the \dpfsaAcr{} entropy.)
A more nuanced picture, however, arises in the empirical relationship of alphabet size and DPFSA rank with overall learning difficulty (Fig. \ref{fig:symbols_rank}). 
Finally, Fig. \ref{fig:rank_outout_dim} demonstrates that learning difficulty numerically increases with rank, though a large hidden state size of the neural LM reduce this difficulty to some degree.\looseness=-1

 
To assess the influence of such effects on KL divergence, we implemented a linear regression model for the output of each neural LM. The independent variables were the \dpfsaAcr{} complexity parameters, as well as the neural LM's hidden state size $\hiddDim$.
The full list of predictors is shown in \cref{tab:axes}.
Prior to model fitting, each predictor was standardized using a $z$-score transformation for an interpretable comparison of the estimated coefficients.\looseness=-1





\subsection{RNN Findings}\label{sec:results}

\begin{table}
    \centering
    \fontsize{10}{10}\selectfont
    \sisetup{table-format=3.2, group-minimum-digits=3}
    \begin{tabular}{lrrr}
        \toprule
        Predictor & $\widehat{\beta}$ & SE & $p$-value \\
        \midrule
        Intercept & $13.5$ & $0.10$ &  $<0.001$ \\
        $\nstates$ & $0.60$ & $0.19$ & $<0.01$ \\
        $\nsymbols$ & $2.10$  & $0.24$ & $<0.001$ \\
        $\nstates \nsymbols$ & $0.64$ & $0.16$ & $<0.001$ \\
        $\pfsaRank$ & $2.99$ & $0.13$ & $<0.001$ \\
        Exp. len. & $11.70$ & $0.24$ & $<0.001$ \\
        $\min(\nstates, \nsymbols)$ &  $-1.36$ & $0.33$ & $<0.001$ \\
        $\entropy(\wfsa)$ & $-7.89$ & $0.28$ & $<0.001$ \\
        $\hiddDim$ & $-2.76$ & $0.10$ & $<0.001$ \\
        \bottomrule
    \end{tabular}
    \caption{Estimated coefficients ($\widehat{\beta}$), standard errors (SE), and $p$-values for $\KL$ generated with a linear regression model for Transformers.}
    \label{tab:Transformers_coefficients}
   \vspace{-2.5pt}
\end{table}
As shown in \cref{tab:rnns_coefficients}, the KL divergence was significantly influenced by the number of states, \dpfsaAcr{} rank, expected string length, \dpfsaAcr{} entropy, and RNN hidden state size ($\nstates$, $\pfsaRank$, Exp. len., $\entropy(\wfsa)$, $\hiddDim$). 
The number of symbols, number of transitions, and the upper bound of $\pfsaRank$ did not reach significance ($\nsymbols$, $\nstates \nsymbols$, $\min(\nstates, \nsymbols)$).
Of the significant effects, the number of states, expected length, and rank were positive in direction, indicating an increase in KL divergence with an increase in the predictor of interest. Perhaps surprisingly, the \dpfsaAcr{} entropy was negative in influence, indicating a decrease in KL with an increase in \dpfsaAcr{} entropy. 
Unsurprisingly, the RNN hidden state size was also negative, indicating a decrease in KL with an increase in hidden state size.
Overall, the \dpfsaAcr{} rank had the strongest influence on KL divergence, followed by the expected length,
the \dpfsaAcr{} entropy, number of states, and RNN hidden state size.
The remaining predictors were smaller in influence.\looseness=-1

\subsection{Transformer Findings}
The linear regression model of the Transformers data revealed significant effects of all included predictors on KL divergence (see \cref{tab:Transformers_coefficients}). 
Of these, the number of states, number of symbols, number of transitions, \dpfsaAcr{} rank, and expected string length were positive in influence, indicating an increase in KL divergence as the predictor of interest increased ($\nstates$, $\nsymbols$, $\nstates \nsymbols$, $\pfsaRank$, Exp. len.). 
The upper bound of $\pfsaRank$, \dpfsaAcr{} entropy, and the Transformer hidden state size were negative in influence, indicating a decrease in KL divergence with an increase in the predictor of interest ($\entropy(\wfsa)$, $\hiddDim$). 
Of the positive relationships, the expected string length had the largest influence, followed in order by \dpfsaAcr{} rank ($\pfsaRank$), then number of symbols ($\nsymbols$).
Of the negative relationships, \dpfsaAcr{} entropy had the largest influence, followed by the Transformer hidden state size, then by the upper bound of $\pfsaRank$.

\begin{figure*}
    \centering
    \includegraphics[width=0.99\textwidth]{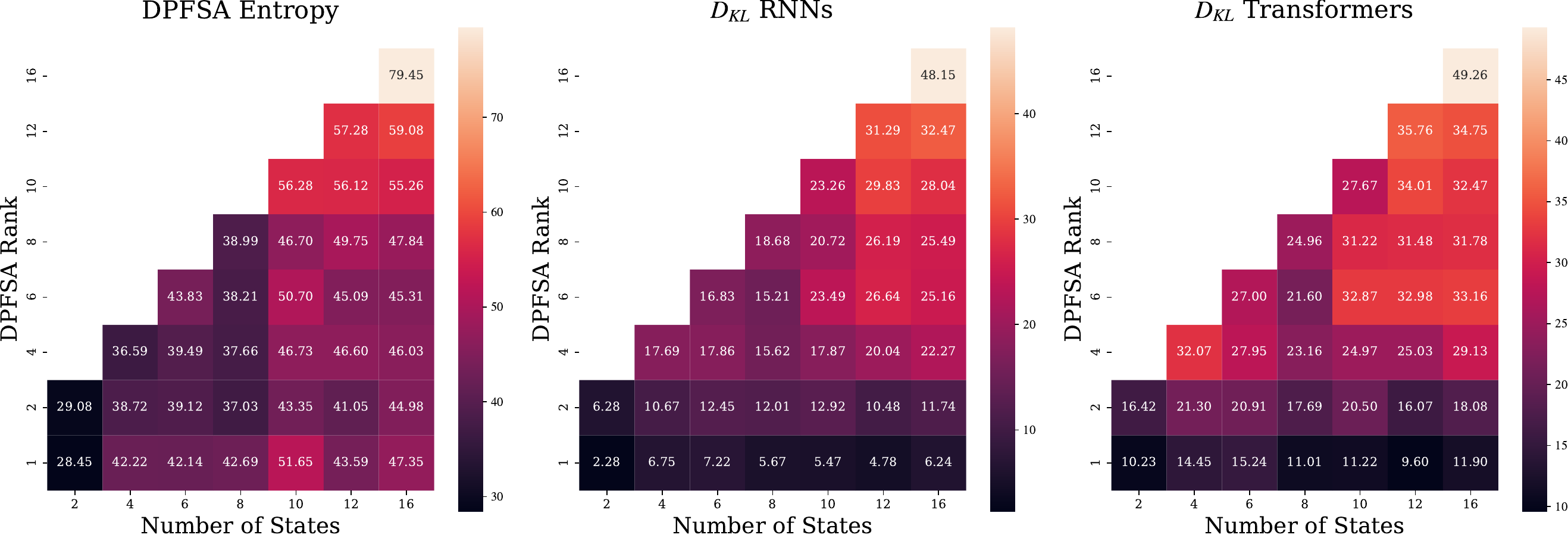}
    \caption { \dpfsaAcr{}'s entropy $\entropy$ and the $\KL$ between the neural LMs and the \dpfsaAcr{}s (in bits) as a function of $\nstates$ and $\pfsaRank$.}
    \label{fig:rank_states}
\end{figure*}
   
\begin{figure*}
    \centering
    \includegraphics[width=0.99\textwidth]{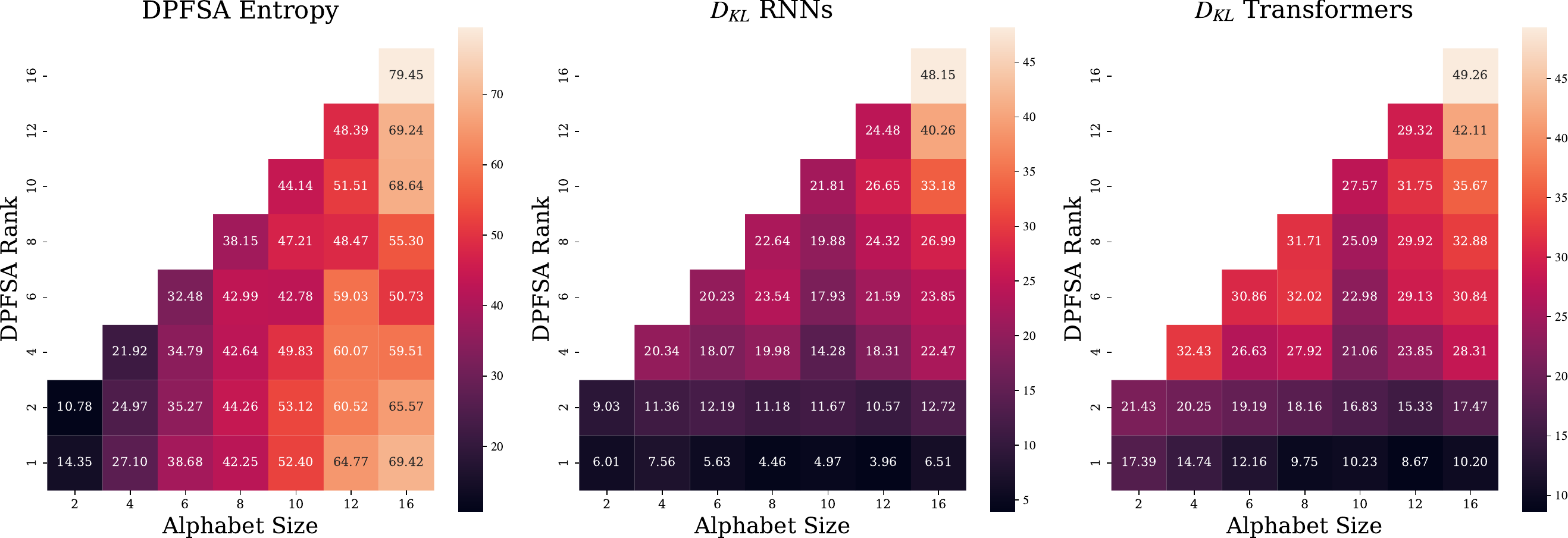}
    \caption{ \dpfsaAcr{}'s entropy $\entropy$ and the $\KL$ between the neural LMs and the \dpfsaAcr{}s (in bits) as a function of $\nsymbols$ and $\pfsaRank$.
    \setlength{\belowcaptionskip}{-15pt}
    }
    \label{fig:symbols_rank}
\end{figure*}
   
\begin{figure*}
    \centering
    \includegraphics[width=0.99\textwidth]{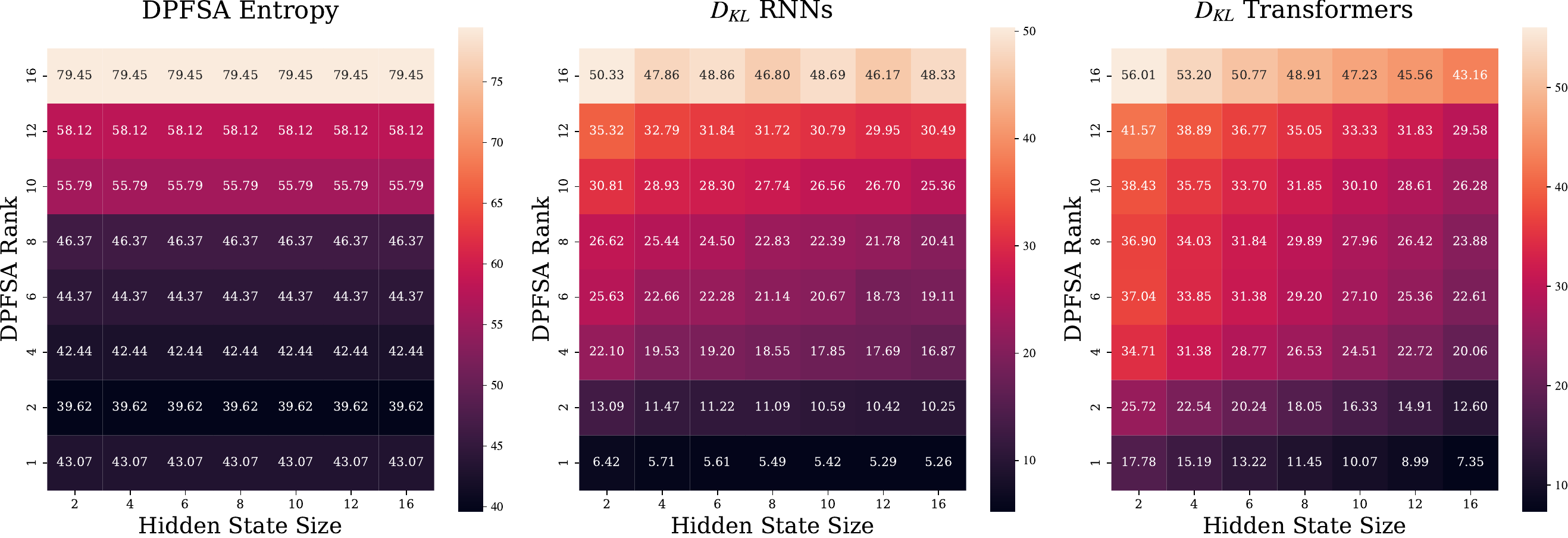}
    \caption{\dpfsaAcr{}'s entropy $\entropy$ and the $\KL$ between the neural LMs and the \dpfsaAcr{}s (in bits) as a function of $\hiddDim$ and $\pfsaRank$.}
    \label{fig:rank_outout_dim}
\end{figure*}

\subsection{Comparing RNN and Transformer LMs}
The linear models revealed an overall similar pattern of effects between RNNs and Transformers, albeit with few notable disparities. 
First, note the overall performance of each neural LM, as indicated by the model intercept for KL divergence---see \Cref{tab:rnns_coefficients}, \cref{tab:Transformers_coefficients}, and \cref{fig:general_performance} in \Cref{app:additional_results}. The model intercept reflects the predicted KL divergence when all other ($z$-transformed) predictors are equal to 0.
RNNs tend to outperform Transformers in this task (RNNs: $\widehat{\beta_0} = 8.72$, Transformers: $\widehat{\beta_0} = 13.5$), demonstrating lower average loss. This difference in performance could be attributed to two main factors: 1) As previous research has shown, RNNs are better suited to modeling formal languages \citep{deletang2023neural}, and 2) Transformers necessitate careful training involving language-specific hyperparameter tuning, which poses a severe computational challenge. 
Nevertheless, we expect that the trend observed here would persist even if the Transformers were trained under optimal conditions. 

While different in magnitude, the number of states, rank, expected string length, \dpfsaAcr{} entropy, and hidden state size all shared the same direction in their influence on KL divergence for each of the RNN and Transformer outputs. Of these, the biggest disparity in magnitude is in the influence of the expected string length and \dpfsaAcr{} entropy, which were significantly larger for Transformers compared to RNNs. Finally, not all predictors reached significance for both RNNs and Transformers. For instance, the number of symbols, number of transitions, and the upper bound of $\pfsaRank$ were significant predictors of KL divergence for Transformers, but not for RNNs.\looseness=-1

As shown in Figs. \ref{fig:rank_states}, \ref{fig:symbols_rank}, and
\ref{fig:rank_outout_dim}, the observed trends in the \dpfsaAcr{}s' entropy, visualized alongside the neural LMs $\KL$, suggest that the entropies alone cannot solely account for the observed effect in $\KL$. 
In other words, larger $\KL$ values are not fully explained by greater entropy values, highlighting the significance of the rank $\pfsaRank$ on learnability even more clearly. Particularly notable is \Cref{fig:rank_outout_dim}, which presents similar patterns to those observed in \Cref{tab:rnns_coefficients} and \Cref{tab:Transformers_coefficients}. The KL divergence of both model types increases with an increase in the \pfsaAcr{} rank and decreases with an increase in the hidden state size. While somewhat noisier, \Cref{fig:rank_states} and \Cref{fig:symbols_rank} further demonstrate the weaker influence of the number of states and number of symbols on the KL divergence.\looseness=-1

\section{Discussion}
\paragraph{Implications of \cref{thm:rank}.}
\cref{thm:rank} concretely quantifies the minimum size of the representation space of \emph{any} neural LM required for the correct representation of regular LMs. 
While this is special case of the so-called \emph{softmax bottleneck} principle \citep{chang-mccallum-2022-softmax}, it is, to the best of our knowledge, the first result connecting the principle to formal models of computation.
Practical implementations of \fslmAcr{}s might use state spaces and alphabets of sizes ranging from tens of thousands to hundreds of thousands \citep{mohririley}, which is much larger than the representations used by most modern neural LMs, which tend to be in the order of a few thousand dimensions \citep{groeneveld2024olmo}.
The high performance of much smaller neural LMs on similar datasets indicates that those LMs are indeed low-rank and can thus be approximated well using smaller hidden representations.
Nevertheless, \cref{thm:rank} provides an interesting limitation on what distributions neural LMs of finite size can represent and points out the limitations of parameter sharing in representing formal models of computation; while neural LMs are good at approximating such models of computation, their inability to represent them \emph{exactly} implies that, with increasing string lengths, their errors will unavoidably accumulate.
This leads to poor length generalization often observed in prior work \citep{weiss-etal-2018-practical,suzgun-etal-2019-evaluating,bhattamishra-etal-2020-ability,deletang2023neural}.

\paragraph{Takeaways from the empirical results.}

The empirical results in \cref{sec:results} complement the theoretical discussion from \cref{sec:theory} and the growing field of literature characterizing the representational capacity of neural LMs.
In line with the theoretical setting and in contrast to related work, our approach directly evaluates the KL divergence between neural LMs and \dpfsaAcr{}s, instead of relying on classification or next-token prediction accuracy measures. 
Comparing distributions over strings offers a more holistic view of a neural LM's overall ability to emulate \dpfsaAcr{}s allowing us to provide compelling insights into what aspects of distributions affect the learnability of formal LMs by controlling for various properties of the \dpfsaAcr{}s being learned.
Neatly, the observed effects of the rank and the neural LM hidden state size on KL divergence align with the theoretical results derived in \cref{thm:rank}: Namely, as the rank of a \dpfsaAcr{} grows, a larger hidden state is required in the neural LM to model the \dpfsaAcr{}'s language appropriately. The dependence of the performance on the rank of the \dpfsaAcr{} demonstrates the utility of formal language theory in providing interpretable insights into the learning abilities of neural LMs.

Another noteworthy observation is the negative influence of the \dpfsaAcr{} entropy on the KL divergence, indicating that neural LMs tend to better learn \fslmAcr{}s with higher intrinsic randomness. A possible explanation for this counter-intuitive result is that when a \dpfsaAcr{} is truly random, its simulation is trivial. However, it is more challenging to accurately learn \fslmAcr{}s with more underlying latent structures.\looseness=-1

\paragraph{Extensions.}
We focus on the learnability of deterministic \pfsaAcr{}s. 
This makes the theoretical results from \cref{sec:theory} particularly interpretable. 
Extensions to the non-deterministic automata, however, are an interesting next step.
Note that in this case, the \pfsaAcr rank analysis is slightly more nuanced.
A non-deterministic \pfsaAcr can, at any point, be in any of the $\nstates$ states (with a probability distribution over them), meaning that the probability of the next symbol is a convex combination of the individual conditional probability distributions (not their logits). 
This makes the analysis trickier and less interpretable; we leave it for future work to make the current exposition more concise.
A further interesting follow-up is also the study of the learnability of (deterministic) \emph{context-free} LMs represented by probabilistic pushdown automata (PPDAs).
PPDAs augment \pfsaAcr{}s by implementing a stack that gives the automaton infinitely many configurations.
Despite the infinitely many configurations, controlling for their rank analogously to the rank of a \pfsaAcr could elucidate how efficiently they are representable by neural LMs.

\section{Conclusion}
We provide a comprehensive empirical study of the learnability of \dpfsaAcr{}s by neural LMs.
More concretely, we investigate how well RNN and Transformer LMs learn to match the distributions over strings generated by \dpfsaAcr{}s of varying complexity. 
To this end, we operationalized learning difficulty using the KL divergence between such distributions over strings; this measure holistically captures the similarity between the neural LM and the targeted regular LM.
In an regression analysis of KL divergence variation, several complexity parameters of the \dpfsaAcr{} reached significance, along with the hidden state size of the neural LM. 
Of note were the \dpfsaAcr{} rank and the expected length of sampled strings, which corresponded to significant increases in KL divergence across the RNN and Transformer results.
We establish that for equivalence, a neural LM's hidden state size is theoretically lower-bounded by the \dpfsaAcr{}'s rank.
This is consistent with the results of our controlled experiments. 
Our results demonstrate that using probabilistic formal language theory can help us generate insights into what neural LMs can and cannot learn.
Nevertheless, our findings call for further theoretical investigations closer to practical applications.\looseness=-1

\section*{Limitations}
We point out some limitations of the presented study.
First and foremost, the finite-state automata considered in this paper are limited by their bounded memory.
Additionally, to keep our work concise and results self-contained, we focus only on deterministic \dpfsaAcr{}s. 
Similar and more comprehensive investigations could of course include non-deterministic automata and languages higher up on the Chomsky hierarchy, such as context-free LMs, or even context-sensitive LMs. 
Our experiments also omit the effect of training dataset size, which might be an interesting quantity to consider when training neural LMs.
We leave those considerations to future work.\looseness=-1

Moreover, due to computational constraints and the substantial computation load imposed by our experiments, we could not fine-tune our models with language-specific hyperparameters, which are particularly important for transformers. 
For the same reason, we had to refrain from optimizing larger and more capable models. 
However, we believe that this should not impair the validity of our results, as the trend we observed would hold even with optimal training.\looseness-1


\section*{Acknowledgements}
Ryan Cotterell acknowledges support from the Swiss National Science Foundation (SNSF) as part of the ``The Forgotten Role of Inductive Bias in Interpretability'' project.
Anej Svete is supported by the ETH AI Center Doctoral Fellowship.
Josef Valvoda is funded by the Nordic Programme for Interdisciplinary Research Grant 105178 and the Danish National Research Foundation Grant no. DNRF169.
This research was partially funded by a DFF Sapere Aude research leader grant under grant agreement No 0171-00034B, the Danish-Israeli Study Foundation in Memory of Josef and Regine Nachemsohn, and the Privacy Black \& White project, a UCPH Data+ Grant. This work was further supported by the Pioneer Centre for AI, DNRF grant number P1.

\bibliography{anthology,anthology_p2,custom}
\bibliographystyle{acl_natbib}

\newpage

\onecolumn

\appendix

\section{Additional Related Work}
\label{app:related_work}

\subsection{Representational Capacity of Neural LMs}
Much of the theoretical work has investigated the representational capacity of various neural LM architectures \citep{MerrillBlackBox,strobl2023transformers}. 
Regular languages (and, to a lesser extent, regular LMs) have been linked to neural LMs particularly often, especially to recurrent neural LMs, but similar connections have also been made for Transformers \citep{merrill-2019-sequential,merrill-etal-2020-formal,liu2023transformers}.
Distinctively interesting are the bounds on the space requirements for emulating FSAs \citep{Dewdney1977,Indyk95,hewitt-etal-2020-rnns,svete2023recurrent}.
This work bridges the theoretical work with practice, tests its applicability, and uses its insights for an informed study of the practical representational capacity of neural LMs.\looseness=-1

\subsection{Learning Formal Languages}
Work similar to ours in spirit is that of \citet{jumelet-zuidema-2023-transparency}, where the authors train and evaluate neural LMs with probabilistic context-free grammars.
They use the underlying data-generating distribution (the probabilistic grammar) to evaluate how well the model has learned the distribution.
Moreover, the knowledge of grammar allows them to probe the model for the encodings of individual constituents, similar to how we probe for the states of the automaton.
In contrast to our work, however, \citet{jumelet-zuidema-2023-transparency} focus on learning human-language-based grammars, which do not provide a holistic picture of the representability of general formal LMs by neural LMs.

\citet{deletang2023neural} provide a comprehensive survey of the learnability of diverse formal languages.
Unlike us, they focus on learning discrete languages, particularly from the perspective of learning \emph{algorithms} and investigating LMs' inductive biases.
They formulate this as a \emph{transduction}---a string-to-string mapping.
They arrive at interesting results showing that popular neural LMs are hard to place on the standard Chomsky hierarchy of languages.
This can partly be explained by the mismatch of the training task---transduction---and the probabilistic nature of a neural LM, since the probabilistic Chomsky hierarchy is known to differ from the discrete one \citep{ICARD2020102308}.
In contrast to our work, \citet{deletang2023neural} also only consider a limited set of hand-picked languages which, while providing algorithmic insights into how LMs work, do not extensively probe the learnability of the language classes.

Testing the compositional generalization of NNs, \citet{valvoda-etal-2022-benchmarking} sample an infinite number of finite languages.
Thereby they can draw conclusions about the learnability of an entire class of languages---sub-regular ones encoded by subsequential finite state transducers.
Their work connects Montague's theory of compositional generalization \cite{montague-1970-universal} with the popular SCAN benchmark of compositional behavior \cite{pmlr-v80-lake18a}.
Unlike our work, they investigate deterministic transducers and seq2seq models.

Another similar work is that of \citet{white-cotterell-2021-examining}, who use artificial languages to identify the biases of neural LMs.
By modifying a base grammar, they experiment with the learnability of 64 languages.
Unlike us, their work focuses solely on topological aspects of the language, which limits their findings to observations over the word order.

In a different line of work, \citet{akyürek2024incontext} evaluate neural LMs' abilities to learn regular languages \emph{in context}.
Rather than learning one particular distribution from the training dataset, they train neural LMs to model the language of any finite-state automaton given a number of samples from it---that is, to infer the generating mechanism from the context.
They consider only discrete languages (even though their generative setup is probabilistic) and due to the in-context learning setting, they do not analyze the dynamics of the neural LM implementing individual languages.

\section{Probabilistic Finite-state Automata} \label{sec:pfsas}
We begin by more formally defining the notion of probabilistic finite-state automata (\pfsaAcr{}s), which were only informally introduced in \cref{sec:preliminaries}.

\begin{definition}\label{def:stochastic-wfsa}
    A \defn{probabilistic finite-state automaton} (\pfsaAcr{}) is a 5-tuple $\wfsatuple$ where $\alphabet$ is an alphabet, $\states$ a finite set of states, $\trans \subseteq \states \times \alphabet \times \left[0, 1\right] \times \states$ a finite set of weighted transitions and $\initf, \finalf\colon \states \rightarrow \left[0, 1\right]$ the initial and final weighting functions.
    Moreover, $\trans, \initf$ and $\finalf$ are required to satisfy that $\sum_{\stateq \in \states} \initf\left(\stateq\right) = 1$, and, for all $\stateq \in \states$, $\sum_{\left(\stateq, \sym, w, \stateq'\right) \in \trans} w + \finalf\left(\stateq\right) = 1$.
We denote $\left(\stateq, \sym, w, \stateq'\right) \in \trans$ with $\edge{\stateq}{\sym}{w}{\stateq'}$.
\end{definition}

\begin{definition}
    A \defn{path} $\apath$ in a $\pfsaAcr$ $\wfsa$ is a string of consecutive transitions
    $\edge{\stateq_1}{\sym_1}{w_1}{\stateq_2}, \cdots, \edge{\stateq_{\pathlen}}{\sym_{\pathlen}}{w_{\pathlen}}{\stateq_{\pathlen + 1}}$.
    A path $\apath$'s \defn{length} $|\apath|$ is the number of transitions in it and its \defn{scan} $\yield\left(\apath\right)$ the concatenation of the symbols on them.
    We denote with $\paths(\automaton)$ the set of all paths in $\automaton$ and with $\paths(\automaton, \str)$ the set of all paths that scan $\str \in \kleene{\alphabet}$.\looseness=-1
\end{definition}

The weights of the transitions along a path are multiplicatively combined to form the weight of the path.
The weights of all the paths scanning the same string are combined additively to form the weight of that string.
\begin{definition}
    The \defn{path weight} of $\apath \in \paths(\automaton)$ is $\weight\left(\apath\right) = \initf \left( \stateq_1 \right) \left[\prod_{\idx = 1}^\pathlen w_\idx\right] \finalf \left( \stateq_{\pathlen + 1} \right)$.
    The \defn{stringsum} of $\str \in \kleene{\alphabet}$
    is $\automaton \left( \str \right) \defeq \sum_{\apath \in \paths\left( \automaton, \str \right) }  \weight \left( \apath \right)$.
\end{definition}
It is easy to see that the final weights $\finalfFun{\stateq}$ play an analogous role to the $\eos$ symbol in the context of autoregressive LMs---they both correspond to the probabilities of ending the generation of the string.

\begin{definition} \label{def:fsa-deterministic}
    A \pfsaAcr $\automaton = \wfsatuple$ is \defn{deterministic} if $|\set{\stateq \mid \initfFun{\stateq} > 0}| = 1$ and, for every $\stateq \in \states, \sym \in \alphabet$, there is at most one $\stateq^\prime \in \states$ such that $\edge{\stateq}{\sym}{w}{\stateq^\prime} \in \trans$ with $w > 0$.
\end{definition}

In general, there can be infinitely many \pfsaAcr{}s that define a given FSLM.
However, in the deterministic case, there is a unique minimal \dpfsaAcr.
\begin{definition}
    A \dpfsaAcr $\wfsa = \wfsatuple$ is \defn{minimal} for the FSLM $\pLM$ if there is no equivalent \dpfsaAcr $\wfsa^\prime = \left(\alphabet, \states^\prime, \initf^\prime, \finalf^\prime, \trans^\prime\right)$ with $|\states^\prime| < |\states|$.\looseness=-1
\end{definition}

\section{Proofs of Theoretical Results} \label{sec:proofs}

\minSizeThm*
\begin{proof}
This proof is effectively a restatement of the softmax bottleneck \citep{yang2018breaking}.
Because the language model $\pLM$ has full support, we may transform it into a softmax-normalized, representation-based LM. 
By supposition, after this transformation, $\pLM$'s output matrix $\transMtx$, as in \cref{sec:softmax-dpfsa}, is of rank $\pfsaRank$.
For $\qLM$ to be equivalent to $\pLM$, it has to hold that
\begin{equation}
        \softmaxFun{\outMtx \hiddState(\str)}{} = \softmaxFun{\transMtx \hiddState_\automaton(\str)}{}, \label{eq:softmax-equality}
\end{equation}
for all $\str \in \kleene{\alphabet}$ where $\outMtx$'s $\qLM$'s output matrix.
Next, define
\begin{subequations}
    \begin{align}
        Z_1(\str) &  \defeq  \sum_{\eossym \in \eosalphabet} \exp((\outMtx \hiddState(\str))_{\eossym}), \\
        Z_2(\str) &\defeq \sum_{\eossym \in \eosalphabet} \exp((\transMtx \hiddState_\automaton(\str))_{\eossym}).
    \end{align}
\end{subequations} 
For all $\eossym \in \eosalphabet$, due to the additive invariance property of the softmax, we have the following


\begin{subequations}
\begin{align}
     \softmaxFun{\outMtx \hiddState(\str) - \log Z_1(\str) \cdot \ones }{\eossym} 
     &= \exp(\outMtx \hiddState(\str) - \log Z_1(\str) \cdot  \ones )_{\eossym} \\
     &= \exp(\transMtx \hiddState_\automaton(\str) - \log Z_2(\str) \cdot \ones )_{\eossym} \\
     &= \softmaxFun{\transMtx \hiddState_\automaton(\str) - \log Z_2(\str) \cdot \ones }{\eossym},
\end{align}
\end{subequations}
where $\ones$ is the vector of all ones $\in \R^{|\eosalphabet|}$. 
It follows that, by taking $\log$s, we have
\begin{equation}\label{eq:inside-exp-equality_1}
    \outMtx \hiddState(\str) - \log Z_1(\str) \cdot \ones = \transMtx \hiddState_\automaton(\str) - \log Z_2(\str) \cdot \ones,
\end{equation}
and, furthermore, that
\begin{equation}\label{eq:inside-exp-equality_2}
    \outMtx \hiddState(\str) = \transMtx \hiddState_\automaton(\str) + \log \frac{Z_1(\str)}{Z_2(\str)} \cdot \ones.
\end{equation}
From \cref{eq:inside-exp-equality_2}, it follows that
\begin{equation}\label{eq:first-span-equality}
\mathspan\left(\{\transMtx \hiddState_{\automaton}(\str) + \log \frac{Z_1(\str)}{Z_2(\str)} \cdot \ones \mid \str \in \kleene{\alphabet} \}\right) = \mathspan\left(\{\outMtx \hiddState(\str) \mid \str \in \kleene{\alphabet} \}\right).
\end{equation}
Noting that $\mathspan\left(\{\log \frac{Z_2(\str)}{Z_1(\str)} \cdot \ones \mid \str \in \kleene{\alphabet}\}\right) = \mathspan(\{ \ones \})$, we can write
\begin{equation}\label{eq:second-span-equality}
\mathspan\left(\{\transMtx \hiddState_{\automaton}(\str) \mid \str \in \kleene{\alphabet} \}\right) \oplus \mathspan\left( \{ \ones \}\right)  = \mathspan\left(\{\outMtx \hiddState(\str) \mid \str \in \kleene{\alphabet} \}\right),
\end{equation}
where $\oplus$ is the direct sum, and we note that $\dim(\mathspan\left(\{\transMtx \hiddState_{\automaton}(\str) \mid \str \in \kleene{\alphabet} \}\right)) = \pfsaRank$ exactly because $\automaton$'s minimality implies $\dim(\mathspan\left(\{ \hiddState_{\automaton}(\str) \mid \str \in \kleene{\alphabet} \}\right)) = \nstates$.
Next \cref{eq:second-span-equality}, in turn, implies
\begin{equation}
\underbrace{\mathspan\left(\{\transMtx \hiddState_{\automaton}(\str) \mid \str \in \kleene{\alphabet} \}\right)}_{\text{dimensionality } \pfsaRank} \subseteq \mathspan\left(\{\outMtx \hiddState(\str) \mid \str \in \kleene{\alphabet} \}\right).
\end{equation}
Thus, we arrive at a lower bound on the rank of $\outMtx$.
Why is it a lower bound? Because for an arbitrary representation-based LM, we do not know $\dim\left(\mathspan\left(\{\hiddState(\str) \mid \str \in \kleene{\alphabet}\}\right)\right)$.
However, we do know it is lower-bounded by $\dim\left(\mathspan\left(\{\outMtx \hiddState(\str) \mid \str \in \kleene{\alphabet} \}\right)\right)$.
Thus, $\rankFun{\outMtx} \geq \pfsaRank$.
Because $\outMtx$'s rank is bounded above by $\hiddDim$, we have $\hiddDim \geq \pfsaRank$, as desired.
\end{proof}




\section{Experimental Details} \label{sec:experimental-details}

\subsection{Sampling \dpfsaAcr{}s of varying complexity}
\label{app:sampling_pdfas}

The \dpfsaAcr{}s we used in our experiments were sampled with $\nstates \in \left\{2, 4, 6, 8, 10, 12, 16\right\}$ over alphabets alphabets of sizes $\nsymbols \in \left\{2, 4, 6, 8, 10, 12, 16\right\}$. Given a sampled \dpfsaAcr{} $\wfsa$ with $\nstates$ states over an alphabet $\alphabet$, we randomly set its unweighted transition function. That is, for each $\stateq \in \states$ and $\sym \in \alphabet$ we randomly choose $\stateq{}' \in \states$ and add the transition $\edgenoweight{\stateq}{\sym}{\stateq{}'}$ to $\wfsa$.

We add weights to the transition function of $\wfsa$ as follows. We generate a random matrix $\transMtx \in \R^{|\eosalphabet| \times \nstates}, \transMtx_{i,j} \sim \mathcal{N}(\mu=0,\,\sigma^{2}=4)\,$, and define $\pfsaRank{}_{\max} = \rankFun{\transMtx}$ (note that $\pfsaRank{}_{\max} \le \min(\nstates, \nsymbols)$). For each $\pfsaRank \in \left\{r \mid r \le \min\left(\nstates, \nsymbols\right), r \in \{1, 2, 4, 6, 8, 10, 12, 16\} \right\}$, we define $\transMtx^\pfsaRank \defeq \text{SVD}\left(\transMtx, \pfsaRank\right)$, where $\text{SVD}\left(\transMtx, \pfsaRank\right)$ is the operation of reducing the rank of $\transMtx$ to $\pfsaRank$ using SVD. Next, we add weights to the transition function of $\wfsa$ by replacing each unweighted transition $\edgenoweight{\stateq}{\sym}{\stateq{}'}$ with $\edge{\stateq}{\sym}{\weight_{\stateq, \sym}}{\stateq{}'}$, where $\weight_{\stateq, \sym} = \softmaxFun{\transMtx{}^\pfsaRank{}_{:, \stateq}}{\sym}$. Finally, we set $\finalf \left( \stateq \right) = \softmaxFun{\transMtx{}^{\pfsaRank{}}_{:, \stateq}}{\eos{}}$. This process results with the generation of up to eight\footnote{Note that we have $|\left\{r \mid r \le \min\left(\nstates, \nsymbols\right), r \in \{1, 2, 4, 6, 8, 10, 12, 16\} \right\}| \le 8$.} random \dpfsaAcr{}s, all sharing the same $\states, \alphabet$ and underlying transition function. They differ, however, in the rank of the matrix $\transMtx{}^\pfsaRank{}$ that defines the weights of the transitions. 

As the Transformer LMs have a limited context length, it is undesirable to sample \dpfsaAcr{}s that are likely to generate strings longer than the context length. 
This may unfairly hinder the Transformer LMs ability to learn such languages.
We, therefore, filter out \dpfsaAcr{}s with an expected string length larger than the median value of expected string lengths (in practice, 46 symbols), i.e., half of the \dpfsaAcr{}s. 
See \Cref{app:evaluation_setup} for details about how the expected generation length of \dpfsaAcr{}s is calculated.

\subsection{Generating the Data}
\label{app:data_generation}

\begin{figure}[t]
    \centering
    \includegraphics[width=0.5\columnwidth]{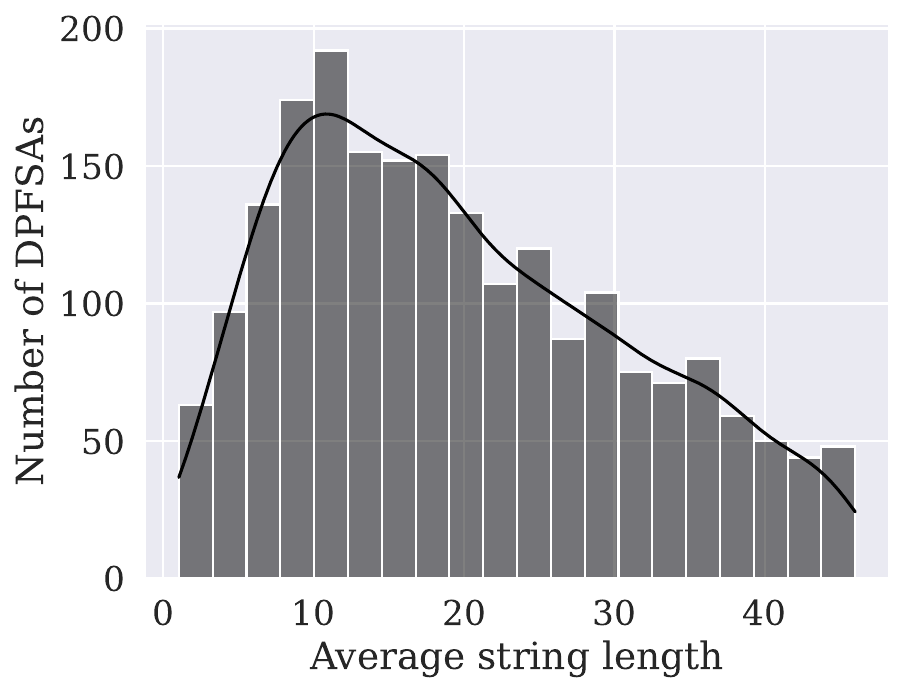}
    \caption{The statistics of the training dataset.}
    \label{fig:data-stats}
\end{figure}

For a given \dpfsaAcr{} $\wfsa$, we sample 20k random strings, terminating the generation process of each string when \eos{} is sampled. We divide the dataset into train and test splits, such that no string is shared between the sets, and the test set has at least 2k strings. We truncate the strings to $256$ symbols to accommodate the limited context length of the Transformer model we used. \Cref{fig:data-stats} shows a histogram of the average length of strings generated for each \dpfsaAcr{}.

\subsection{Training the Neural LMs}

We train neural LMs on our dataset using the following procedure, repeated $2100$ times:
\begin{enumerate}[itemsep=2pt]
    \item Sample a random \dpfsaAcr{} $\wfsa$ with $\nstates$, $\nsymbols$ and rank $\pfsaRank$ using the process described in \Cref{app:sampling_pdfas}.
    \item Sample 20k strings from $\wfsa$ and split them to train set and test set using the process described in \Cref{app:data_generation}. 
    \item For each $\hiddDim \in \left\{2, 4, 6, 8, 10, 12, 16\right\}$ train an RNN on the train set strings.
    \item For each $\hiddDim \in \left\{2, 4, 6, 8, 10, 12, 16\right\}$ train a Transformer model on the train set strings.
    \item Compute the $\KL$ between $\wfsa$ and the trained neural LMs on the test set strings.
\end{enumerate}

We train the RNN and Transformer models using the following hyperparameters:
\begin{itemize}[itemsep=1pt]
    \item \textbf{RNNs}: We use a unidirectional LSTM with four hidden layers, each with $64$-dimensional hidden states and an embedding size of $64$. We trained each model for two epochs using a batch size of $32$ and a learning rate of $0.001$, an Adam optimizer with default settings, and a standard cross-entropy loss \citep{kingma2014adam}. We did not tie the weights of the word embeddings.
    \item \textbf{Transformers}: We use the GPT-2 model architecture \cite{radford2019language} with six attention layers, each with four attention heads and $128$-dimensional representations. We use an embedding size of $64$ and an input context length of $256$. We trained each model for two epochs using a batch size of $32$ and a learning rate of $0.001$, an AdamW \citep{loshchilov2018fixing} optimizer with default settings, and a standard cross-entropy loss. We did not tie the weights of the word embeddings.
\end{itemize}

\subsection{Derivations} 
\label{app:evaluation_setup}

\subsubsection{Entropy of a \dpfsaAcr{}} 
Let $\wfsa = \wfsatuple$ be a softmax-normalized \dpfsaAcr{} 
with the parameters matrix $\transMtx \in \R^{|\eosalphabet| \times \nstates}$.
The entropy of $\wfsa$ is defined as 
\begin{equation} \label{eq:ent-defn}
\entropy\left(\wfsa\right)  \defeq - \sum_{\str \in \kleene{\alphabet}} \automaton \left( \str \right) \log \automaton \left( \str \right).
\end{equation}
\citet[Thm. 4.2]{grenander1967syntax} provides an efficient method of computing $\entropy\left(\wfsa\right)$, which we summarize in the following, based on \citeposs{sanchez-etal-2018-derivational} exposition. 

In deterministic \pfsaAcr{}s, \cref{eq:ent-defn} can equivalently be computed by summing over the \emph{paths} in $\wfsa$, $\paths(\automaton)$, rather than the strings $\str \in \kleene{\alphabet}$; as $\wfsa$ is deterministic, there exists a single path $\apath$ in $\paths(\automaton, \str)$ for every $\str \in \kleene{\alphabet}$. 
We can then derive
\begin{subequations}
\begin{align}
\entropy\left(\wfsa\right) &  = - \sum_{\str \in \kleene{\alphabet}} \automaton \left( \str \right) \log \automaton \left( \str \right) \\
& = - \sum_{\apath \in \paths(\automaton)} \weight \left( \apath \right) \log \weight \left( \apath \right).
\end{align}
\end{subequations}
Now, let $\tranMtx \in \R^{\nstates \times \nstates}$ be the matrix of transition probabilities between each two states.
That is, $\tranMtx_{i,j}$ is the probability of transitioning from state $\stateq_i \in \states$ to state $\stateq_j \in \states$ in $\wfsa$, and it is computed as 
\begin{equation} \label{eq:M-mtx}
    \tranMtx_{i,j} \defeq \sum_{\sym \in \alphabet}  \ind{\edge{\stateq_i}{\sym}{w}{\stateq_j} \in \trans} \cdot \softmaxFun{\transMtx_{:, \stateq_{i}}}{\sym},
\end{equation}
where $\ind{\edge{\stateq_i}{\sym}{w}{\stateq_j} \in \trans}$ is the indicator function of whether the transition $\edge{\stateq_i}{\sym}{w}{\stateq_j}$ exists in $\wfsa$.\footnote{$\tranMtx$ together with a column of final weights is the stochastic matrix describing $\wfsa$, since each row contains the probabilities of emitting any of the symbols or stopping the generation.}
Similarly, let $\valpha \in \R^{\nstates}$ be defined such that $\valpha_i \defeq \initf(\stateq_i)$, i.e., the probability of starting a generation at state $\stateq_i$.
Finally, let $\vxi \in \R^{\nstates}$ be defined as
\begin{equation} \label{eq:xi}
    \vxi_{i} \defeq - \sum_{\sym \in \eosalphabet} \softmaxFun{\transMtx_{:, \stateq_i}}{\sym} \log \softmaxFun{\transMtx_{:, \stateq_i}}{\sym}.
\end{equation}
This allows us to restate \citet[Thm. 4.2]{grenander1967syntax}.
\begin{theorem}[Entropy of a \dpfsaAcr{}]
    Given the definitions in \cref{eq:M-mtx,eq:xi}, $\entropy\left(\wfsa\right)$ can be computed using the following expression\looseness=-1
    \begin{align} \label{eq:entropy-expression}
        \entropy\left(\wfsa\right) = \valpha^{\top} \left(\mI - \tranMtx\right)^{-1} \vxi.
    \end{align}
\end{theorem}
We use \cref{eq:entropy-expression} to compute the entropy of all our regular LMs.

\subsubsection{Expected String Length Under a \dpfsaAcr{}} 
To compute the expected length of a string under a \dpfsaAcr{}, we rely on the following lemma.
\begin{lemma}[Expected string length under a \dpfsaAcr{}]
    The expected length of strings generated by a \dpfsaAcr{} $\wfsa = \wfsatuple$ can be computed using the following identity
    \begin{equation}
        \underset{\apath \in \paths}{\E}\left[\pathlen(\apath) \right] = \sum_{\stateq \in \states} \left(\valpha^{\top} \left(\mI - \tranMtx\right)^{-1} \right)_{\stateq} - 1.
    \end{equation}
\end{lemma}
\begin{proof}
The following derivation proves the result
    \begin{subequations}
        \begin{align}
            \sum_{\stateq \in \states} \left(\valpha^{\top} \left(\mI - \tranMtx\right)^{-1} \right)_{\stateq} 
            &= \sum_{\stateq \in \states} \left(\valpha^{\top} \left(\sum_{\idx = 0}^{\infty} \tranMtx^{\idx}\right)\right)_{\stateq} \\
            & = \sum_{\stateq \in \states} \left[\sum_{\idx = 0}^{\infty} \left(\valpha^{\top}\tranMtx^{\idx}\right)_{\stateq}\right] \\
            & = \sum_{\idx = 0}^{\infty} \left[\sum_{\stateq \in \states} \left(\valpha^{\top}\tranMtx^{\idx}\right)_{\stateq} \right].
        \end{align}
    \end{subequations}
    Notice that $\left(\valpha^{\top}\tranMtx^{\idx}\right)_{\stateq}$ is the total weight of paths in $\wfsa$ reaching the state $\stateq$, but does not necessarily \emph{terminate} there due to the absence of final weights $\finalf$, after exactly $\idx$ transitions, starting in initial states according to their initial weights \citep{malagutti2024role}.  
    It follows that $\sum_{i \in \states} \left(\valpha^{\top}\tranMtx^{\idx}\right)_i$ equals the weights of all paths (visiting in any state $\in \states$) that generate a string of length of \emph{at least} $\idx$:
    \begin{equation}
        \sum_{i \in \states} \left(\valpha^{\top}\tranMtx^{\idx}\right)_i = \wfsa\left(\set{\str \mid |\str| \geq k}\right).
    \end{equation}
    Let $p(\pathlen(\apath) = k\mid \wfsa)$ be the probability of the $\wfsa$ generating a string of length \emph{exactly} $k$, terminating the generation after $k$ transitions. 
    In other words, $p(\pathlen(\apath) = k\mid \wfsa)$ corresponds to the probability mass of the set $\set{\str \mid |\str| = k}$ under the \dpfsaAcr{}.
    By the reasoning above, 
    \begin{equation}
        \sum_{\stateq \in \states} \left(\valpha^{\top}\tranMtx^{\idx}\right)_{\stateq} = \sum_{k = \idx}^{\infty} p(\pathlen(\apath) = k \mid \wfsa),
    \end{equation}
    and
    \begin{subequations}
        \begin{align}
            \sum_{\idx = 0}^{\infty} & \left[\sum_{\stateq \in \states} \left(\valpha^{\top}\tranMtx^{\idx}\right)_{\stateq} \right] \\
            & = \sum_{\idx = 0}^{\infty} \sum_{k = \idx}^{\infty} p(\pathlen(\apath) = k \mid \wfsa) \\
            & = \sum_{k = 0}^{\infty} p(\pathlen(\apath) = k \mid \wfsa) + \sum_{k = 1}^{\infty} p(\pathlen(\apath) = k \mid \wfsa) + \sum_{k = 2}^{\infty} p(\pathlen(\apath) = k \mid \wfsa) + \ldots \\
            & = p(\pathlen(\apath) = 0 \mid \wfsa) + p(\pathlen(\apath) = 1 \mid \wfsa) + p(\pathlen(\apath) = 2 \mid \wfsa) + p(\pathlen(\apath) = 3 \mid \wfsa) + \cdots\\
            & \phantom{\, \, =p(\pathlen(\apath) = 0 \mid \wfsa)} + p(\pathlen(\apath) = 1 \mid \wfsa) + p(\pathlen(\apath) = 2 \mid \wfsa) + p(\pathlen(\apath) = 3 \mid \wfsa) + \cdots \nonumber \\
            & \phantom{\, \, =p(\pathlen(\apath) = 0 \mid \wfsa) + p(\pathlen(\apath) = 1 \mid \wfsa)} + p(\pathlen(\apath) = 2 \mid \wfsa) + p(\pathlen(\apath) = 3 \mid \wfsa) + \cdots \nonumber \\
            & \phantom{\, \, =p(\pathlen(\apath) = 0 \mid \wfsa) + p(\pathlen(\apath) = 1 \mid \wfsa) + p(\pathlen(\apath) = 2 \mid \wfsa)} + p(\pathlen(\apath) = 3 \mid \wfsa) + \cdots \nonumber \\
            & \phantom{\, \, =p(\pathlen(\apath) = 0 \mid \wfsa) + p(\pathlen(\apath) = 1 \mid \wfsa) + p(\pathlen(\apath) = 2 \mid \wfsa) + p(\pathlen(\apath) = 3 \mid \wfsa)} + \cdots \nonumber \\
            & = \sum_{\idx = 0}^{\infty} p(\pathlen(\apath) = n \mid \wfsa) + \sum_{\idx = 0}^{\infty} \idx \cdot p(\pathlen(\apath) = n \mid \wfsa) \\
            & = 1 + \sum_{\idx = 0}^{\infty} \idx \cdot p(\pathlen(\apath) = n \mid \wfsa) \\
            & = 1 + \underset{\apath \in \paths}{\mathds{E}}\left[\pathlen(\apath) \right].
        \end{align}
    \end{subequations}
\end{proof}

\subsubsection{Cross Entropy Between a \dpfsaAcr{} and a Neural LM} 
To stochastically approximate the cross entropy between a ground-truth \dpfsaAcr{} and a neural LM, we compute 
\begin{equation}
    \begin{aligned}
    \widehat{\entropy}\left(\pLM, \qLM\right)  = -\frac{1}{\datasetSize}\sum_{n=1}^{\datasetSize}\left( \log\qLM(\eos \mid \str^{(n)}) + \sum_{t=1}^{\mymacro{|\str_n|}}\log\qLM(\sym^{(n)}_t \mid \strlt^{(n)})\right),
\end{aligned}    
\end{equation}
where $\qLM(\overline{\sym}_t \mid \strlt^{(n)}) \defeq \softmaxFun{\outMtx \hiddStatetminus}{\overline{\sym}_t}$ is given by the neural LM.
The strings $\str^{(n)}$ are sampled i.i.d. from $\pLM$, the language model induced by the \dpfsaAcr{}.


\section{Additional Results}
\label{app:additional_results}

This section includes figures presenting the results of additional experiments augmenting and supporting the claims in the main paper.

\paragraph{\Cref{fig:general_performance}.} Overall performance of the neural LM models as a function of the ``total complexity'' of the \pfsaAcr{} they were optimised for, which we define as the sum of $\nsymbols$, $\nstates$, and $\pfsaRank$. Performance is measured as the cross-entropy of the neural model on a held-out test set. We compute the loss by summing it over symbols and dividing this sum by the number of strings in the test set. 
We can see that RNNs tend to overperform Transformers, especially for more complex \pfsaAcr{}s.

\paragraph{\Cref{fig:kl_vs_length}.} The $\KL{}$ of RNNs and Transformers as a function of the average string length of the \dpfsaAcr. Similarly to \Cref{tab:Transformers_coefficients}, we see that Transformers are much more sensitive to string length compared to RNNs.

\begin{figure}[t]
    \centering
    \includegraphics[width=0.7\columnwidth]{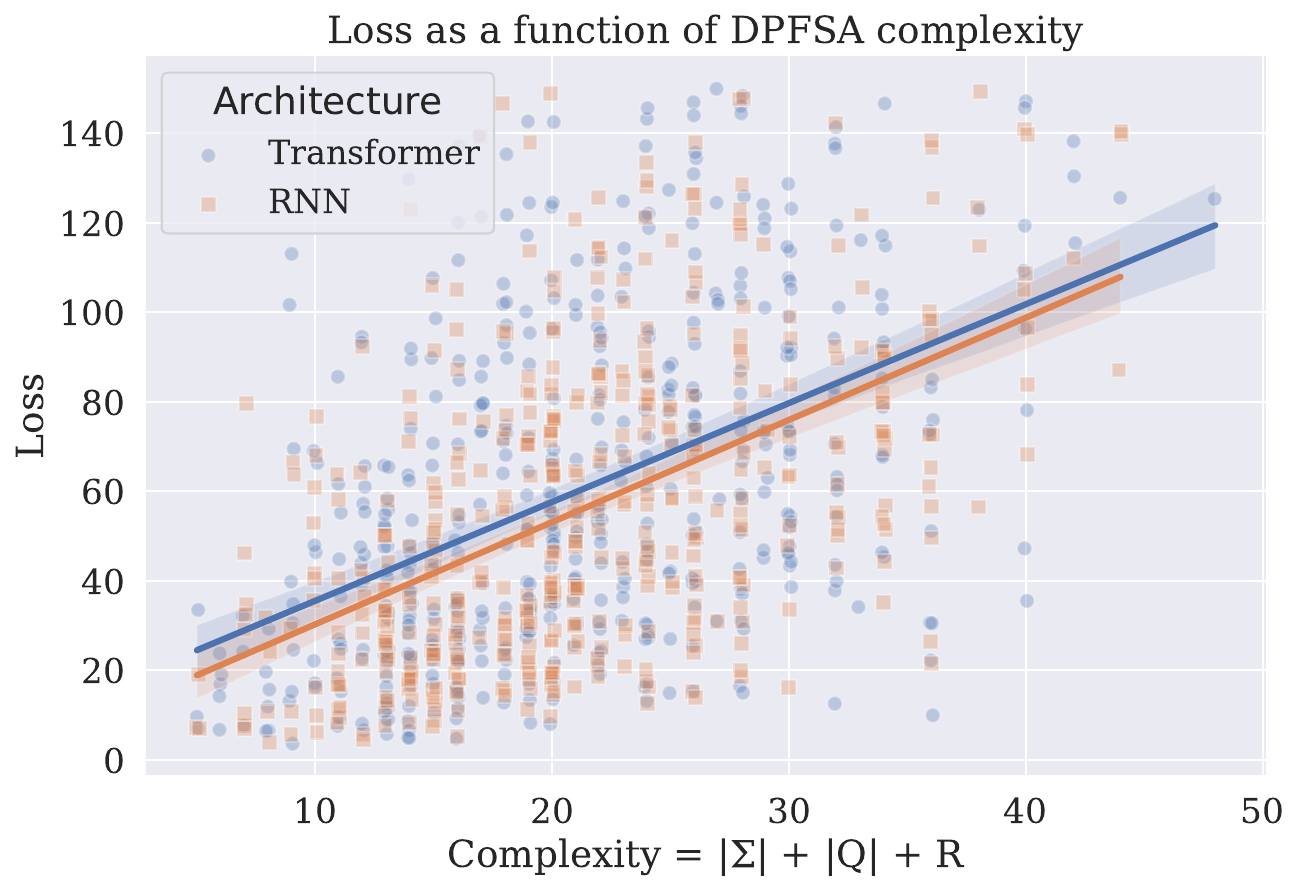}
    \caption{Validation performance of RNNs and Transformers as a function of the PFSA's complexity, computed as $\nsymbols + \nstates + \pfsaRank$. We compute the loss by summing it over symbols and dividing this sum by the number of strings in the test set. }
    \label{fig:general_performance}
\end{figure}

\begin{figure}[t]
    \centering
    \includegraphics[width=0.7\columnwidth]{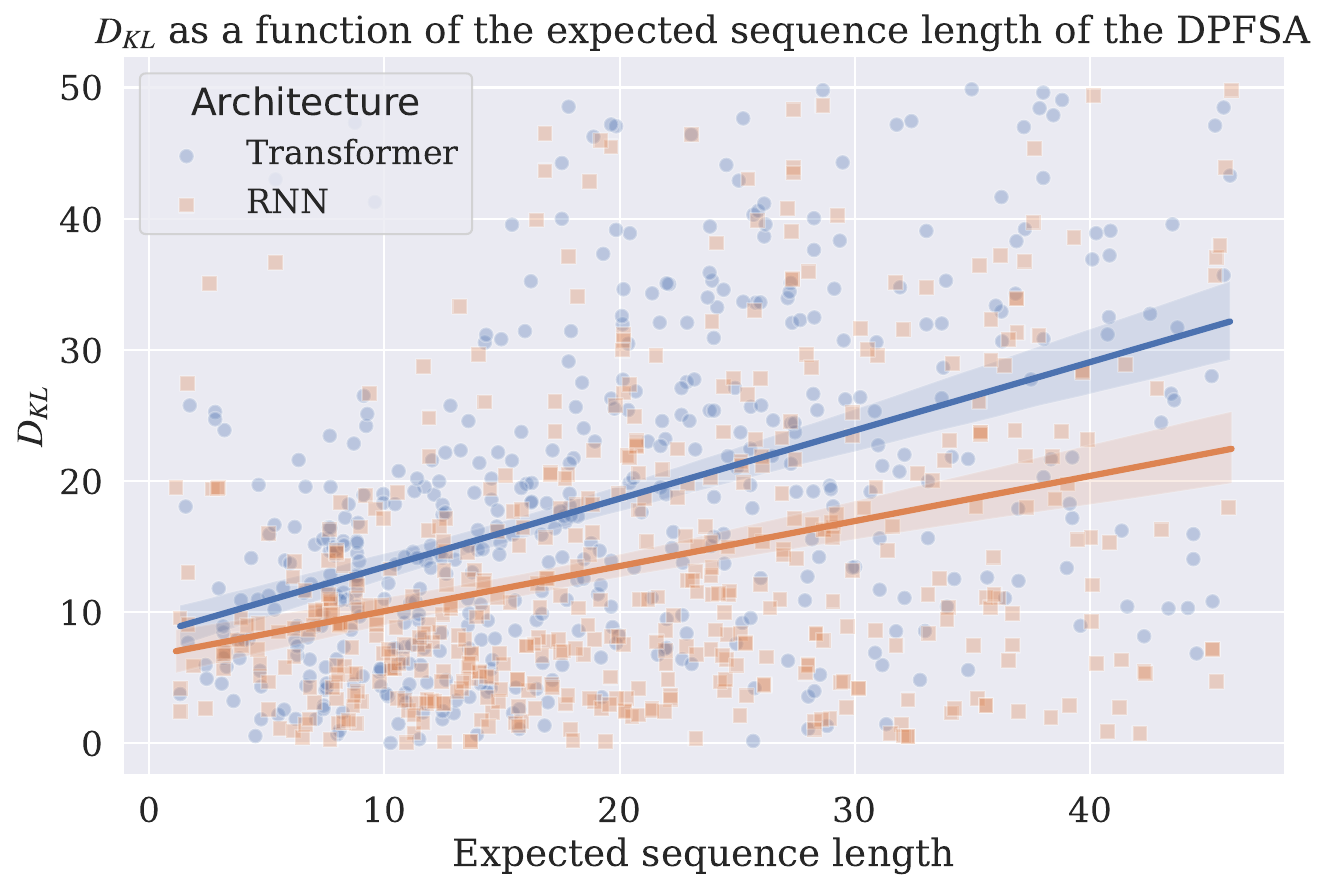}
    \caption{$\KL{}$ of RNNs and Transformers as a function of the average string length of the \dpfsaAcr.}
    \label{fig:kl_vs_length}
\end{figure}

\end{document}